%% file: neurips_2026.tex
\documentclass{article}

    \PassOptionsToPackage{authoryear, round}{natbib}
 \usepackage[preprint]{neurips_2026}

\usepackage[utf8]{inputenc} 
\usepackage[T1]{fontenc}    
\usepackage{hyperref}       
\usepackage{url}            
\usepackage{booktabs}       
\usepackage{amsfonts}       
\usepackage{nicefrac}       
\usepackage{microtype}      
\usepackage{xcolor}         
\usepackage{graphicx}
\usepackage{amsmath}
\usepackage{amssymb}
\usepackage{mathtools}
\usepackage{amsthm}
\usepackage{tabularx}
\usepackage{multirow}

\theoremstyle{plain}
\newtheorem{theorem}{Theorem}[section]
\newtheorem{proposition}[theorem]{Proposition}
\newtheorem{lemma}[theorem]{Lemma}
\newtheorem{corollary}[theorem]{Corollary}
\theoremstyle{definition}

\theoremstyle{remark}

\newcommand{\fm}{\text{FM}}
\newcommand{\cfm}{\text{CFM}}
\newcommand{\vfm}{\text{VFM}}
\newcommand{\moefm}{\text{MoE-FM}}

\newcommand{\nar}{\text{NAR}}
\newcommand{\ar}{\text{AR}}
\newcommand{\yan}{\textsc{YAN}}
\newcommand{\data}{\text{data}}
\newcommand{\lx}{{L'}}
\newcommand{\ly}{{L}}
\newcommand{\vocab}{\mathcal{V}}
\newcommand{\src}{\text{src}}
\newcommand{\tgt}{\text{tgt}}
\newcommand{\ce}{\text{CE}}
\newcommand{\mmd}{\text{MMD}}
\newcommand{\scale}{\text{Scale}}
\newcommand{\kl}{\text{KL}}
\newcommand{\llada}{\text{LLaDA}}
\newcommand{\diffullama}{\text{DiffuLLaMA}}
\usepackage{bbm}

\title{Towards Faster Language Model Inference Using Mixture-of-Experts Flow Matching}

\author{%
  Aihua Li \\
  Duke University\\
  \texttt{aihua.li@duke.edu} \\
}

\begin{document}

\maketitle

\begin{abstract}
Flow matching retains the generation quality of diffusion models while enabling substantially faster inference, making it a compelling paradigm for generative modeling. However, when applied to language modeling, it exhibits fundamental limitations in representing complex latent distributions with irregular geometries, such as anisotropy and multimodality. To address these challenges, we propose a mixture-of-experts flow matching (MoE-FM) framework, which captures complex global transport geometries in latent space by decomposing them into locally specialized vector fields. Building on MoE-FM, we develop a non-autoregressive (NAR) language modeling approach, named YAN, instantiated with both Transformer and Mamba architectures. Across multiple downstream tasks, YAN achieves generation quality on par with both autoregressive (AR) and diffusion-based NAR language models, while requiring as few as three sampling steps. This yields a $40\times$ speedup over AR baselines and up to a $10^3\times$ speedup over diffusion language models, demonstrating substantial efficiency advantages for language modeling. 
\end{abstract}

\section{Introduction}

Along with the remarkable success of autoregressive (\ar) large language models in generating high-quality text \citep{gpt2_2019, llama2024, qwen_technical_report, deepseek2024}, their inference latency has long been a subject of concern. \ar\ models generate tokens sequentially in a left-to-right manner, requiring one forward pass per token generation \citep{nar_nmt_gu2018, mask_predict_2019, fast_decoding_using_discrete_latent2018}. Aiming to parallelize decoding and speed up inference, diffusion-based non-autoregressive (\nar) language models have emerged as a popular alternative in recent years, inspired by the remarkable performance of diffusion models in computer vision \citep{ddpm_ho2020, diffusion_beats_gan2021}. However, in the context of language generation, these methods still exhibit a fundamental performance gap compared to well-established \ar\ models at comparable scales \citep{on_the_learning2022, benefit_limitation_diffusionlm2025, tricks_of_trade2021, enriching_nar2021}. In practice, to achieve competitive quality, existing diffusion methods typically require hundreds or thousands of inference steps to iteratively refine generated tokens \citep{benefit_limitation_diffusionlm2025}. This, in turn, offsets the theoretical efficiency benefits of parallel decoding. 

To improve the quality-efficiency trade-off of current \nar\ language models, we investigate flow matching, a generative modeling paradigm that has demonstrated substantial efficiency advantages \citep{lipman_flow_2022, lipman_flow_guide, flow_straight_fast, consistency_fm2024}. Flow matching generates samples by integrating a deterministic ordinary differential equation (ODE), which can be trained to follow relatively straight trajectories, thereby bypassing the iterative denoising procedures of diffusion models that contribute to high inference latency. Beyond efficiency, flow matching has also achieved strong generation quality in image and video synthesis \citep{scaling_rf_2024, flow_matching_video2023}. Despite this promise, its application to language modeling remains largely unexplored. 

However, when instantiated for language modeling, we identify a fundamental limitation of flow matching in modeling complex latent distributions. Prior studies have shown that text representations exhibit highly irregular geometries, including multimodality, anisotropy, and fragmented manifolds \citep{isotropic_cluster_and_manifold2021, cone_gao2019, how_contextual2019, cluster_approach_improve_isotropy2021, fine_tuning_geometry2021}. Under such conditions, vanilla flow matching with a single global vector field proves insufficient to faithfully capture the underlying transport structure, particularly under limited training scales and a small number of sampling steps. 

To enhance the representational capacity of flow-based language models, we propose \textbf{mixture-of-experts flow matching} (\moefm; Fig.~\ref{fig:moe}). \moefm\ models the conditional target vector field through a mixture-of-experts formulation, where multiple expert vector fields are combined via data-dependent soft routing. This approach effectively decomposes global transport and encourages specialization in distinct local transport geometries. Building on this enhanced flow matching formulation, we introduce a new \nar\ language modeling paradigm, which leverages \moefm\ in the latent space, aiming to learn token representations that are sufficiently expressive to support efficient sequence decoding with few parallelizable layers. We refer to the proposed model as \textbf{YAN}---\textit{Flow Until \textbf{Y}ou \textbf{A}lmost K\textbf{n}ow} (Fig.~\ref{fig:yan_overview}). Our main contributions are summarized as follows: 
\begin{itemize}
    \item We improve the representational fidelity and sampling efficiency of flow matching for language modeling by developing mixture-of-experts flow matching (Fig.~\ref{fig:moe}). This leads to higher-quality recovery of text latent representations compared to vanilla flow matching, and induces straighter ODE trajectories that allow accurate generation with fewer integration steps.
    \item We introduce \yan, a non-autoregressive language modeling framework that employs latent mixture-of-experts flow matching (Fig.~\ref{fig:yan_overview}). We instantiate \yan\ with both Transformer \citep{transformer2017} and Mamba \citep{mamba_gu2024} architectures, and train models at the 200M-parameter scale.
    \item We evaluate \yan\ across a range of downstream language modeling tasks, including text infilling and completion, and show that it achieves improved generation quality relative to baseline methods (Tab.~\ref{tab:main_results}). 
    \item We analyze the inference efficiency of \yan\ and show that it achieves high-quality long-document infilling with as few as three Euler sampling steps. This results in a $40\text{--}50\times$ speedup over autoregressive baselines and an order-of-magnitude $10^3\times$ speedup over diffusion-based language models (Fig.~\ref{fig:efficiency_analysis}). 
\end{itemize}

\begin{figure*}[t]
\vskip -0.05in
  \begin{center}
    \includegraphics[width=\textwidth]{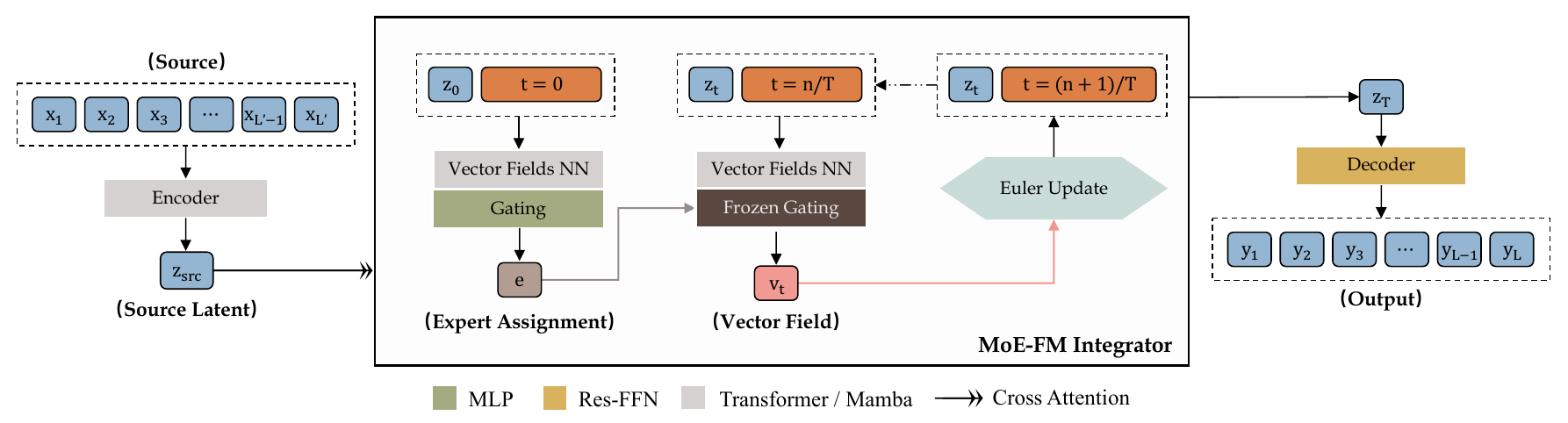}
    \vskip -0.05in
    \caption{Overview of the \yan\ non-autoregressive language model.}
    \label{fig:yan_overview}
  \end{center}
\end{figure*}

\section{Background and Problem Statement}

\begin{figure*}[t]
  \begin{center}
    \centerline{\includegraphics[width=\textwidth]{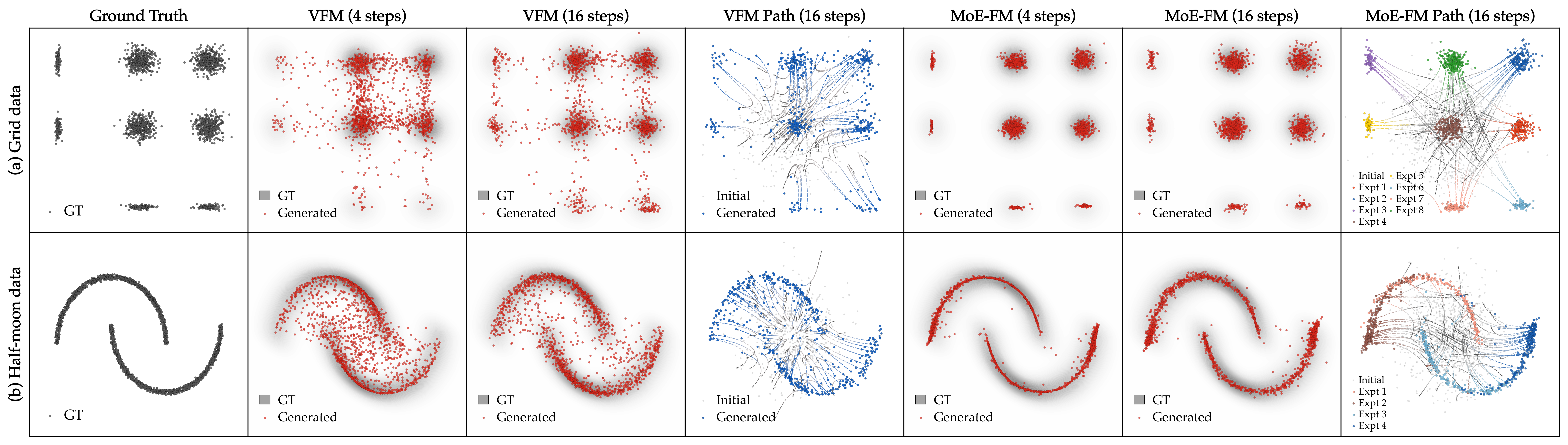}}
    \vskip -0.05in
    \caption{Comparison of vanilla flow matching (\vfm) and mixture-of-experts flow matching (\moefm) on synthetic datasets. Results on (a) grid and (b) half-moon data show that \moefm\ more accurately recovers data distributions with irregular geometries, including disconnected, curved modes and nonlinear low-dimensional manifolds. Moreover, \moefm\ learns straighter transport trajectories from noise to data, thereby improving generative performance with fewer sampling steps.}
    \label{fig:moe}
  \end{center}
\end{figure*}

\subsection{Flow Matching}

Flow matching (\fm) \citep{lipman_flow_2022, lipman_flow_guide, introduction_flow_mit} is a class of generative models that synthesizes samples from a target distribution $p_1$ by learning a time-dependent vector field that transports an initial distribution $p_0$ to $p_1$. Formally, let $u:\mathbb{R}^m\times [0,1]\to \mathbb{R}^m,\,(z_t,t)\mapsto u_t(z_t)$ denote a vector field over $\mathbb{R}^m$ defining the ordinary differential equation (ODE)
\begin{equation*}
    dz_t=u_t(z_t)dt,
\end{equation*}
whose solution is a trajectory $\{z_t\}_{0\le t\le 1}\subset \mathbb{R}^m$. The goal of \fm\ is to learn a target vector field $u^*$ such that, if the initial state satisfies $z_0\sim p_0$, then the terminal state satisfies $z_1\sim p_1$, with intermediate states $z_t$ following a prescribed probability path $\{p_t\}_{0\le t\le 1}$, i.e., $z_t\sim p_t$ for all $t\in(0,1)$. 

Inspired by the forward process of denoising diffusion models \citep{ddpm_ho2020, ddim_song2021, introduction_flow_mit}, \fm\ adopts the Gaussian conditional probability path defined as $p_{t\mid 1}(z\mid z_1)=\mathcal{N}(z;tz_1, (1-t)^2 I)$ with $z_0\sim \mathcal{N}(0,I)$. This is equivalent to imposing a linear interpolation $z_t=tz_1+(1-t)z_0$, yielding the target conditional vector field $u_t^*(z_t\mid z_1)=(z_1-z_t)/(1-t)$. \fm\ learns this target conditional vector field by regressing it with a parameterized vector field $u_t^\psi$ under an $\ell_2$ loss 
\begin{equation}\label{eq:cfm loss}
    \mathcal{L}_\cfm(\psi)=\mathbb{E}_{t,z_1,z_t\mid z_1}\|u_t^\psi(z_t)-u_t^*(z_t\mid z_1)\|^2,
\end{equation}
where the expectation is taken over $t\sim \mathcal{U}(0,1)$, $z_1\sim p_1$, and the intermediate conditional distribution $z_t\mid z_1\sim p_{t\mid 1}$. The validity of this optimization follows from the fact that the conditional objective~\eqref{eq:cfm loss} is equivalent, up to a constant independent of the parameter $\psi$, to the marginal objective $\mathcal{L}_\fm(\psi)=\mathbb{E}_{t,z_t}\|u_t^\psi(z_t)-u_t^*(z_t)\|^2$, where $t\sim \mathcal{U}(0,1)$ and $z_t\sim p_t$. In other words, optimizing the conditional objective~\eqref{eq:cfm loss} equivalently recovers the target marginal vector field \citep{lipman_flow_2022}. Moreover, the Gaussian conditional formulation is commonly used in practice, as it leads to the simplified objective
\begin{equation}\label{eq:vfm loss}
    \mathcal{L}_\vfm(\psi)=\mathbb{E}_{t,z_1,z_t\mid z_1}\|u_t^\psi(z_t)-(z_1-z_0)\|^2,
\end{equation}
which can be efficiently estimated via Monte Carlo sampling. As indicated by Eq.~\eqref{eq:vfm loss}, under Gaussian linear interpolation, the target conditional vector field induces straight-line trajectories connecting the initial noise $z_0$ to the target endpoint $z_1$. Sampling along such deterministic ODE trajectories enables faster inference than diffusion models, which iteratively denoise the samples along stochastic trajectories \citep{lipman_flow_2022, flow_straight_fast, scorebased_song2021}.

\subsection{Non-Autoregressive Language Modeling}

The language modeling objective of this work is formulated as a standard sequence-to-sequence task, which aims to learn the distribution of a target token sequence $y=(y_1, \dots, y_\ly)$ given a source sequence $x=(x_1, \dots, x_\lx)$, with each token drawn from a vocabulary $\vocab=\{1, \dots, V\}$. Motivated by the efficiency advantages of flow matching, we aim to leverage flow matching to push the efficiency limits of current \nar\ language models \citep{llada2025, mdlm2024, diffulamma2025, dream2025}. We adopt a latent variable \nar\ formulation
\begin{equation}\label{eq:latent variable model}
    p(y\mid x)=\int p(y\mid z,x)p(z\mid x)dz
\end{equation}
with a latent variable $z\in\mathbb{R}^m$ \citep{delta_posterior2020, nar_nmt_gu2018, tricks_of_trade2021, seqdiffuseq2022}, and apply continuous flow matching in the latent space. Compared to discrete flow matching \citep{discrete_flow_matching2024}, the continuous approach typically exhibits better optimization stability and has shown higher quality in prior work \citep{alpha_flow2025}.

\section{Flow Matching for Language Modeling}\label{sec:flow_matching}

We begin by showing the limitations of vanilla flow matching when applied to latent language modeling (Sec.~\ref{sec:limitations_of_vfm}), and then develop a mixture-of-experts approach (Sec.~\ref{sec:moe_fm}).

\subsection{Limitations of Vanilla Flow Matching}\label{sec:limitations_of_vfm}

In our preliminary experiments, the vanilla flow matching (\vfm) trained with objective~\eqref{eq:vfm loss} underperforms in learning the token latent distribution introduced in Sec.~\ref{sec:nar_lm_via_yan}, particularly under finite training scale. Intuitively, this distribution is highly anisotropic, concentrated on a degenerate manifold, and exhibits isolated modes and clustering effects \citep{isotropic_cluster_and_manifold2021, cone_gao2019, how_contextual2019, cluster_approach_improve_isotropy2021, fine_tuning_geometry2021}. These geometries pose a challenging setting for \vfm. To illustrate this issue, Fig.~\ref{fig:moe} presents two examples---multimodal grid data and half-moon data---demonstrating the performance of \vfm\ under irregular geometries involving disconnected, curved modes and nonlinear low-dimensional manifolds. Samples are generated using an Euler ODE solver. As shown, \vfm\ produces samples that are blurred across modes, and this further degrades as the number of sampling steps decreases. Moreover, the learned trajectories are highly curved, despite the target vector field being designed to be straight. Similar issues have been reported in recent works \citep{efficient_fm2025, variational_fm2025}.

From a theoretical standpoint, this limitation can be attributed to the fact that \vfm\ approximates the true distribution of the target vector field $q_\data(u^*\mid z_t,t)$ with a Gaussian distribution, $q_\psi^\vfm(u^*\mid z_t,t)=\mathcal{N}(u^*;u_t^\psi(z_t),I)$, by minimizing the Kullback-Leibler (\kl) divergence \citep{kl1951}
\begin{equation*}
    \kl\big(q_\data\|q_\psi^\vfm\big)=\mathbb{E}_{u^*\sim q_\data(u^*\mid z_t,t)}\big[\log q_\data(u^*\mid z_t,t)-\log q_\psi^\vfm(u^*\mid z_t,t)\big].
\end{equation*}
Hereafter, for notational simplicity, $u^*=z_1-z_0$ denotes the sample-level conditional target vector field in Eq.~\eqref{eq:vfm loss}. This objective is equivalent to minimizing the $\ell_2$ objective~\eqref{eq:vfm loss}, whose solution is characterized by the following proposition:
\begin{proposition}\label{proposition}
$\hat{u}^\vfm(z_t,t)=\mathbb{E}[z_1-z_0\mid z_t, t]$ is the conditional minimizer of the \vfm\ objective~\eqref{eq:vfm loss}.
\end{proposition}
See proof in Appendix~\ref{append:derivation1}. Since the Gaussian approximation is unimodal and, in the presence of multimodality, its solution follows the average direction of $z_1-z_0$, it tends to underfit multimodal vector fields, especially under few-step sampling and limited training data. Motivated by the intuition that multimodality can be better captured by employing multiple vector fields associated with different modes, we incorporate a mixture-of-experts mechanism (MoE; \citeauthor{moe1991}, \citeyear{moe1991}; \citeauthor{moe1994}, \citeyear{moe1994}) and propose mixture-of-experts flow matching in the following section.

\subsection{Mixture-of-Experts Flow Matching}\label{sec:moe_fm}

\textbf{Mixture-of-experts flow matching} (\moefm) aims to improve generative performance and inference efficiency by decomposing global transport into multiple locally specialized vector fields that model heterogeneous transport geometries. Specifically, it introduces $K$ expert vector fields
\begin{equation*}
    u_k^\psi:\mathbb{R}^m\times [0,1]\to \mathbb{R}^m,\;\;(z_t,t)\mapsto u_{k,t}^\psi(z_t),
\end{equation*}
for $k=1, \dots, K$, together with a learnable gating network that outputs $K$ routing probabilities
\begin{equation*}
    \pi^\psi:\mathbb{R}^m\times [0,1]\to \Delta^{K-1},\;\;(z_t,t)\mapsto (\pi_1^\psi, \dots, \pi_K^\psi),
\end{equation*}
where $\Delta^{K-1}$ denotes the $(K-1)$-dimensional probability simplex, meaning that $\pi_k^\psi\ge 0$ and $\sum_{k=1}^K \pi_k^\psi=1$. The true conditional distribution $q_\data(u^*\mid z_t,t)$ is approximated by an MoE model
\begin{equation}\label{eq:moe model}
\begin{aligned}
    q_\psi^\moefm&(u^*\mid z_t,t)=\sum_{k=1}^K\pi_k^\psi(z_t,t)\mathcal{N}(u^*;u_{k,t}^\psi(z_t),\sigma^2 I),
\end{aligned}
\end{equation}
by minimizing the \kl\ divergence $\kl\big(q_\data\|q_\psi^\moefm\big)$, where $\sigma\ge 0$ is a pre-specified parameter discussed below. This is equivalent to minimizing the negative log-likelihood (NLL) loss
\begin{equation}\label{eq:moe_loss}
\begin{aligned}
    \mathcal{L}_\moefm(\psi)=\mathbb{E}_{t,z_1,z_t\mid z_1}\bigg[-\log\sum_{k=1}^K\Big\{\pi_k^\psi(z_t,t) \times\exp\Big(
    -\frac{1}{2\sigma^2}\|u_{k,t}^\psi(z_t)-(z_1-z_0)\|^2\Big)\Big\}\bigg].
\end{aligned}
\end{equation}

\textbf{Properties and Benefits.} To characterize the behavior of \moefm, the following theorem describes the optimal expert vector fields and routing functions.
\begin{theorem}\label{theorem}
The \moefm\ objective~\eqref{eq:moe_loss} admits the following conditional optima:
\begin{equation*}
\begin{aligned}
    \hat{\pi}_k^\moefm(z_t,t)&=\mathbb{E}[\gamma_k^\psi(z_t,t,z_1-z_0)\mid z_t,t],\\
    \hat{u}_k^\moefm(z_t,t)&=\frac{\mathbb{E}[\gamma_k^\psi(z_t,t,z_1-z_0)(z_1-z_0) \mid z_t,t]}{\mathbb{E}[\gamma_k^\psi(z_t,t,z_1-z_0)\mid z_t,t]},
\end{aligned}
\end{equation*}
for $k=1, \dots, K$. Here, $\gamma_k^\psi(z_t,t,z_1-z_0)$ denotes the $k$-th expert responsibility, defined as
\begin{equation*}
\gamma_k^\psi(z_t,t,z_1-z_0)=\frac{\pi_k^\psi \kappa_\sigma\big(z_1-z_0, u_{k,t}^\psi(z_t)\big)}{\sum_{k'=1}^K\pi_{k'}^\psi\kappa_\sigma\big(z_1-z_0, u_{k',t}^\psi(z_t)\big)},
\end{equation*}
where $\kappa_\sigma(v,v')=\exp\left\{-\|v-v'\|^2/(2\sigma^2)\right\}$ is the radial basis function kernel. 
\end{theorem}
The proof is provided in Appendix~\ref{append:derivation2}. These results show that the expert responsibilities $\gamma_k^\psi$ implement a soft gating mechanism that assigns each velocity target $u^*=z_1-z_0$ to experts according to proximity in the vector field space. Each expert vector field $\hat{u}_k^\moefm$ then estimates a local transport direction via $\gamma_k^\psi$-weighted averaging. Consequently, different experts specialize in distinct local flow geometries. As shown in Fig.~\ref{fig:moe}, \moefm\ effectively allocates different expert vector fields to model distinct regions of the space. This produces higher-quality samples that more faithfully recover the data distribution compared to \vfm. Moreover, the interpolating trajectories are substantially straighter, enabling accurate reconstruction of multimodal samples with very few sampling steps (e.g., four steps). 

\textbf{Special Cases.} (1) When $K=1$, \moefm\ reduces to the \vfm\ method, in which $\sigma$ is independent of $\psi$ and therefore does not affect the training objective. (2) The parameter $\sigma\ge 0$ controls the softness of expert assignments. As $\sigma\to 0$, the routing converges to a hard assignment rule using nearest-neighbor assignment. As $\sigma\to\infty$, the likelihood becomes non-identifiable in the expert assignments, meaning that all assignments are equally likely. (see Appendix~\ref{append:derivation3}).

\textbf{Sampling.} We adopt a trajectory-level frozen routing strategy during sampling (Fig.~\ref{fig:yan_overview}). Specifically, we sample a discrete expert assignment $e\sim\mathrm{Cat}(\pi^\psi(z_0,0))$ at time $t=0$, and then keep the assignment fixed throughout the trajectory. The final state is then obtained by integrating the corresponding time-dependent vector field $u_{e,t}^\psi(z_t)$. This yields a trajectory-level conditional generation process, in which each trajectory is governed by a single expert vector field. Compared to time-varying routing strategies, frozen routing avoids frequent expert switching during ODE integration, resulting in much more stable and computationally efficient integration. Moreover, it preserves trajectory-level geometric consistency by preventing continuous interpolation across different transport fields. In our language modeling formulation, we apply the MoE mechanism at the token level, allowing token-specific trajectories.

\section{\nar\ Language Modeling via \yan}\label{sec:nar_lm_via_yan}

We propose a flow-based \nar\ language modeling paradigm (Sec.~\ref{sec:yan_formulation}) and describe its training strategy (Sec.~\ref{sec:training}). The overall pipeline is illustrated in Fig.~\ref{fig:yan_overview}. 

\subsection{Mathematical Formulation}\label{sec:yan_formulation}

\yan\ is a latent variable language model in the class of~\eqref{eq:latent variable model}, parameterized as
\begin{equation}\label{eq:yan}
    p_{\theta,\psi}(y\mid x)=\int p_\theta(y\mid z)p_\psi(z\mid x)dz,
\end{equation}
where $p_\theta(y\mid z)$ is a \textit{decoder} that generates the target sequence $y$ from a continuous latent representation $z\in\mathbb{R}^{\ly\times d}$, and $p_\psi(z\mid x)$ is a \textit{conditional latent generator} given the source sequence $x$. Additionally, \yan\ introduces an \textit{encoder}, $\mathcal{E}_\phi:\vocab^{L^*}\to\mathbb{R}^{L^*\times d}$, mapping discrete sequences of arbitrary length $L^*$ to a continuous space $\mathbb{R}^{L^*\times d}$. 

In \yan, the latent variable $z$ is intended to capture the global semantic content and token-level dependency structure of the target sequence $y$. Conditioning on such a representation is designed to reduce token dependencies and, ideally, render tokens $y_{1:\ly}$ conditionally independent, i.e., 
\begin{equation*}
    p_\theta(y\mid z)=\prod_{l=1}^\ly p_\theta(y_l\mid z).
\end{equation*}
Under this assumption, decoding can be performed in a fully parallel manner where each token is generated independently as
\begin{equation}\label{eq:yan decoder}
    y_l\mid z\sim \mathrm{Cat}\left(\xi_{\theta,l}(z)\right),\quad l=1, \dots, \ly,
\end{equation}
and $\{\xi_{\theta,l}(z)\}_{l=1}^\ly\subset \Delta^{V-1}$ denote token-wise probability vectors. The conditional latent generator $p_\psi(z\mid x)$ is learned using the mixture-of-experts flow matching proposed in Sec.~\ref{sec:flow_matching}. The encoder serves two purposes. First, it contextualizes the source sequence via $z_\src=\mathcal{E}_\phi(x)$. Second, it constructs a target latent representation $z_\tgt=\mathcal{E}_\phi(y)$, which defines the endpoint that the latent generator is trained to flow toward. This design enables self-supervised training of the latent flow. In the absence of such a target, an alternative is to introduce a pretrained teacher for distillation; however, in our preliminary experiments, we observe no empirical benefit from doing so, consistent with findings in other work \citep{diffusionlm_controllable_2022}.

In summary, \yan\ aims to learn a latent representation $z$ that is sufficiently expressive for the target sequence $y$ to be decoded using a very small number of parallelizable layers, while employing flow matching as an efficient conditional latent generator. Motivated by this, we name the model \textbf{YAN}---\textit{Flow Until \textbf{Y}ou \textbf{A}lmost K\textbf{n}ow}.

\subsection{Two-Stage Training}\label{sec:training}

\begin{figure}[t]
  \begin{center}
    \includegraphics[width=0.6\columnwidth]{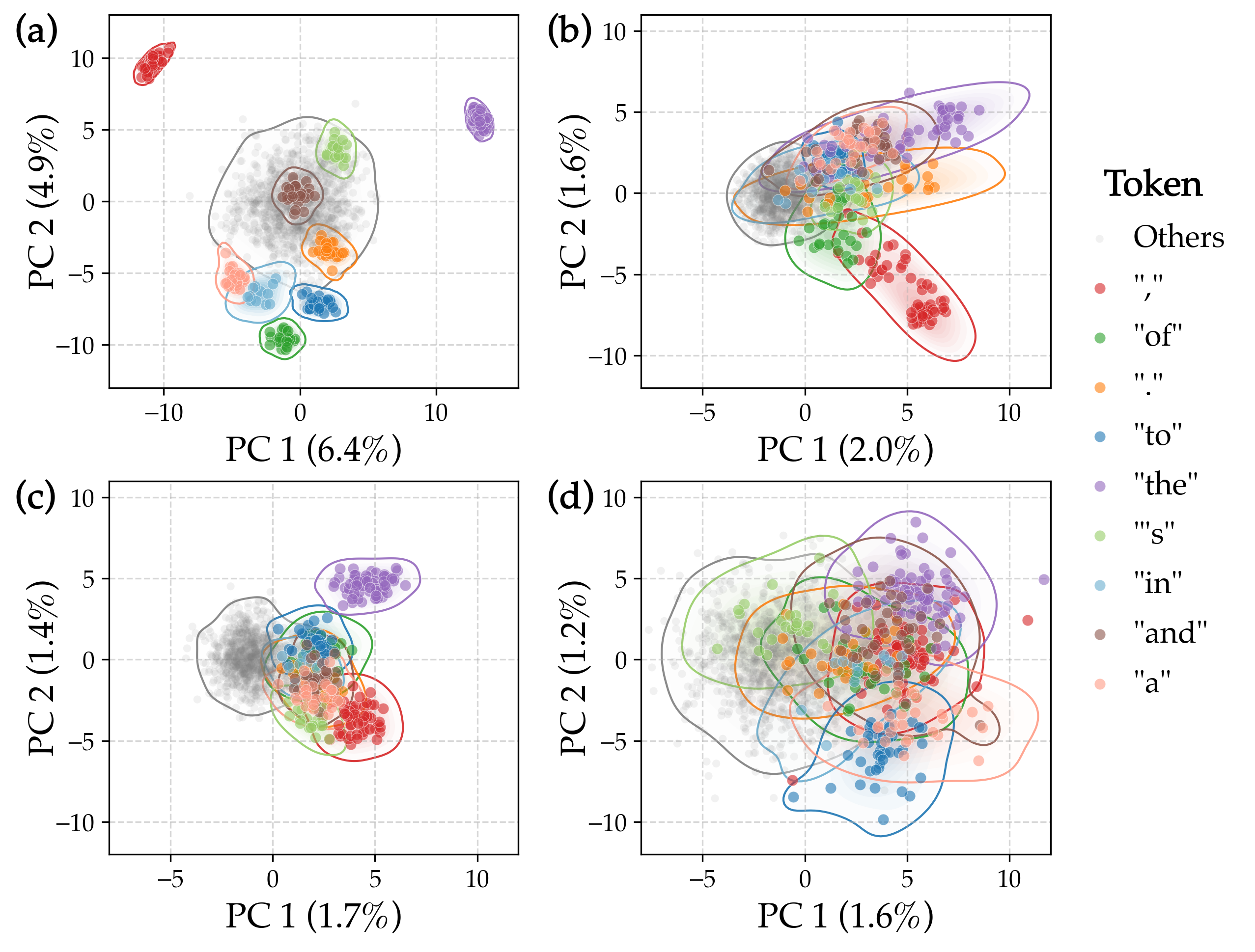}
    \vskip -0.07in
    \caption{Latent representation distributions under different regularization schemes, projected onto the first two principal components and colored by selected high-frequency tokens: (a) unregularized cross-entropy, (b) \mmd, (c) \mmd\ + scale regularization, and (d) \mmd\ + scale regularization with noisy perturbation. Percentages in parentheses denote the variance explained. Reduced variance concentration in the leading components indicates more isotropic latent spaces, with no dominant direction of variation.}
    \label{fig:regularization}
  \end{center}
\end{figure}

Under the latent variable formulation~\eqref{eq:yan}, direct maximum likelihood training is intractable due to the marginalization over the latent variable. As an alternative, we adopt a two-stage training strategy. 

\textbf{Stage 1: Train the Encoder and Decoder.} \yan\ employs an asymmetric autoencoder comprised of a high-capacity encoder and a lightweight decoder. Let $\tilde{y}\in\vocab^{\tilde{L}}$ denote a text sequence from distribution $\tilde{p}_\data$ and $\tilde{z}=\mathcal{E}_\phi(\tilde{y})$ be its encoded latent representation. The first training stage learns to perform reconstructions of $\tilde{y}$ using this autoencoder under objective 
\begin{equation}\label{eq:loss ae}
\begin{aligned}
    \mathcal{L}_{\text{RegularizedAE}}&(\Theta)=\lambda_\ce\mathcal{L}_\ce(\Theta)+\lambda_\mmd\mathcal{L}_\mmd(\Theta)+\lambda_\scale\mathcal{L}_\scale(\Theta).
\end{aligned}
\end{equation}
for pre-specified tuning parameters $\lambda_\ce, \lambda_\mmd, \lambda_\scale \ge 0$, where $\Theta=(\theta,\phi,\psi)$ denotes the collection of learnable parameters for notational simplicity.

Here, the cross-entropy (CE) reconstruction loss,
\begin{equation}\label{eq:loss ce}
    \mathcal{L}_\ce(\Theta)=\mathbb{E}_{\tilde{y}\sim \tilde{p}_\data}\bigg[-\sum_{l=1}^{\tilde{L}}\log p_\theta(\tilde{y}_l\mid \tilde{z})\bigg], 
\end{equation}
is augmented by two latent regularizers. The first regularizer is the Maximum Mean Discrepancy (\mmd; \citeauthor{kernel_two_sample_test2012}, \citeyear{kernel_two_sample_test2012}; \citeauthor{training_mmd2015}, \citeyear{training_mmd2015}), a measure of the discrepancy between the marginal distribution $P$ of $\tilde{z}$ and the standard Gaussian distribution $Q=\mathcal{N}(0,I)$, defined as 
\begin{equation*}\label{eq:mmd definition}
\begin{aligned}
    \mmd_\kappa^2(P,Q)=\mathbb{E}_{x,x'\sim P}[\kappa(x,x')]+\mathbb{E}_{y,y'\sim Q}[\kappa(y,y')]-2\mathbb{E}_{x\sim P, y\sim Q}[\kappa(x,y)],
\end{aligned}
\end{equation*}
where $\kappa(x,x')=\sum_{s\in\mathcal{S}}\kappa_s(x,x')$ is a sum of radial basis function kernels $\kappa_s(x,x')=\exp\left\{-\|x-x'\|^2/(2s^2)\right\}$ over a range of bandwidths $s\in\mathcal{S}$. Replacing the expectations with the corresponding empirical estimates yields the \mmd\ term used in training,
\begin{equation}\label{eq:loss mmd}
    \mathcal{L}_\mmd(\Theta)=\widehat{\mmd}_\kappa^2(P,Q).
\end{equation}
See Appendix~\ref{append:mmd} for details. Our use of the \mmd\ regularizer here is motivated by the Wasserstein autoencoder (WAE; \citeauthor{wasserstein_ae2018}, \citeyear{wasserstein_ae2018}). Although closely related to the variational autoencoder (VAE; \citeauthor{vae_kingma2013}, \citeyear{vae_kingma2013}), WAE outperforms VAE by encouraging the latent representations to match the Gaussian distribution marginally rather than conditionally, thereby improving reconstruction quality \citep{wasserstein_ae2018}. We also apply an $\ell_2$ penalty on the scale of $\tilde{z}$, $\mathcal{L}_\scale(\Theta)=\|\tilde{z}\|^2$, to discourage latent representations from drifting excessively far from the origin.

Theoretically, these regularization choices can be viewed as promoting latent isotropy while preserving local semantic structure, a property shown to benefit downstream performance and facilitate effective utilization of the latent space \citep{shrink_the_longest2024, on_the_isotropy2023}. Following the same principle, we further inject small Gaussian perturbations into the encoder outputs and let the decoder reconstruct from these perturbed representations, thereby expanding the region associated with each latent code. Fig.~\ref{fig:regularization} visualizes the learned latent representation distributions under different regularization schemes. As shown, imposing \mmd\ and scale regularization encourages the distribution to concentrate toward the origin and exhibit more isotropic, Gaussian-like behavior.

\textbf{Stage 2: Train the Latent Flow.} The second stage trains the \moefm\ via the following objective 
\begin{equation}\label{eq:yan_objective}
\mathcal{L}_\yan(\Theta)=\alpha_\moefm\mathcal{L}_\moefm(\Theta)+\alpha_\ce\mathcal{L}_\ce(\Theta),
\end{equation}
where $\mathcal{L}_\moefm$ is the negative log-likelihood loss defined in Eq.~\eqref{eq:moe_loss}. In this stage, the latent flow is trained with initial point $z_0\sim \mathcal{N}(0,I)$ and target endpoint $z_\tgt=\mathcal{E}_\phi(y)$, conditioned on the source latent representation $z_\src=\mathcal{E}_\phi(x)$. After obtaining the final latent state $\hat{z}_1$ via the frozen routing sampling with an Euler solver, $\mathcal{L}_\ce$ is computed as the cross-entropy loss between the ground-truth sequence and the sequence decoded from $\hat{z}_1$. We provide additional analysis of the training in Appendix~\ref{append:ce_supervision}.

\section{Experiments}\label{sec:experiments}

\input{table.tex}

\subsection{Experimental Setup}

Our experiments evaluate the language generation and understanding capabilities of \yan, as well as its inference efficiency, particularly on long-form documents. 

In language model evaluation, while perplexity is a common evaluation metric, we do not use it in our setting for two reasons. First, perplexity has known limitations in measuring long-range understanding and does not always align with human-like language processing \citep{perplexity_limitation_2024, perplexity_limitation2_2021}. Second, theoretically, \nar\ language models---including the latent variable formulation in this work---generally do not admit tractable likelihoods. Instead, likelihoods must be approximated via estimation procedures or variational lower bounds (\citeauthor{mdlm2024}, \citeyear{mdlm2024}; \citeauthor{sedd2024}, \citeyear{sedd2024}), making cross-method perplexity comparisons unreliable. Accordingly, we adopt metrics that directly assess generation quality across multiple downstream tasks, in line with prior \nar\ work \citep{diffusionlm_controllable_2022, diffulamma2025}.

\textbf{Architecture.} We parameterize \yan\ using both Transformer \citep{transformer2017} and Mamba \citep{mamba_gu2024} architectures, as shown in Fig.~\ref{fig:yan_overview}. The Transformer accepts both state and time as inputs, similar to Diffusion Transformers \citep{dit_2023}, and incorporates cross-attention to the source input. It operates bidirectionally without causal masking. Mamba does not support cross-attention by design; therefore, we condition on the source input by combining Mamba-based self token mixing with an explicit cross-attention layer. Under computational constraints, we train \yan\ at the 200M-parameter scale. Detailed hyperparameter settings are provided in Appendix~\ref{append:architecture_details}.

\textbf{Training.} We pretrain \yan\ on a dataset consisting of FineWiki (75\%) \citep{finewiki_data} and FineWeb (25\%) \citep{fineweb_data}. FineWiki contains full-length Wikipedia articles and FineWeb is a large corpus of English web text. See Appendix~\ref{append:training_data} for detailed dataset statistics and descriptions.

\textbf{Tasks and Metrics.} We consider four downstream tasks. (1) \textbf{Text infilling} evaluates the ability to leverage bidirectional global context to generate coherent text for missing spans. We use NarrativeQA \citep{narrativeqa_data} and report ROUGE \citep{rouge2004} and tokens per second (TPS). (2) \textbf{Last-word completion} evaluates next-token prediction given left-to-right context. We evaluate on ROCStories \citep{rocstories_data} and SimpleStories \citep{simplestories2025} using exact match for accuracy and BERTScore-F1 \citep{bertscore2020} for semantic similarity. (3) \textbf{Question answering (QA)} assesses reading comprehension and answer extraction from supporting passages. We evaluate on bAbI \citep{babi_data1, babi_data2} and SQuAD \citep{squad_f1_2016} using F1 score and BERTScore-F1. (4) \textbf{Classification} is a standard benchmark task for language understanding \citep{glue_evaluation}. We use the AG News, DBpedia \citep{agnews_dbpedia_data}, and SST-2 \citep{sst2_data} datasets and report classification accuracy. Additional dataset and evaluation details are provided in Appendix~\ref{append:evaluation_details}. We further evaluate diversity using Dist-2 \citep{dist2_metric2016} for bigram diversity, Self-BLEU \citep{selfblue_metric2018} for sentence-level similarity, and semantic distance (SemDist), computed as the average cosine distance between Sentence-BERT embeddings \citep{semanticdist_metric2019}.

\textbf{Baselines.} For \nar\ methods, we include \llada\ (8B-Base) \citep{llada2025}, which represents the state of the art in performance and scale among diffusion language models, including DiffuSeq (91M) \citep{diffuseq2023}, Plaid (1B) \citep{likelihood_based_diffusion2023}, and MDLM (110M) \citep{mdlm2024}. We also evaluated \diffullama\ (7B) \citep{diffulamma2025}, an \nar\ diffusion model annealed from pretrained LLaMA \citep{llama2024}; however, its performance is consistently inferior to that of \llada, and results are therefore omitted. For \ar\ methods, we include BART (139M) \citep{bart2020}, whose source-target structure is similar to \yan, whereas \llada\ treats the source text as a prefix concatenated to the target. We also include GPT-2 (124M) \citep{gpt2_2019}, a decoder-only \ar\ baseline frequently used in prior \nar\ work (e.g., \citeauthor{mdlm2024}, \citeyear{mdlm2024}; \citeauthor{diffuseq2023}, \citeyear{diffuseq2023}; \citeauthor{diffulamma2025}, \citeyear{diffulamma2025}). These \ar\ models are smaller than \yan, thereby controlling for model size when evaluating inference efficiency.

Since \yan\ is not designed for general-purpose zero-shot language tasks in its current scope, we fine-tune it on each downstream task, following common practice for models with comparable capacity \citep{diffuseq2023, mdlm2024}. For fairness, baseline models are also fine-tuned according to the procedures described in their original papers (see Appendix~\ref{append:finetune} for details). To ensure a consistent comparison of efficiency, we perform inference with a batch size of 1 and use oracle sequence lengths for all methods, thereby eliminating the impact of padding on efficiency measurements. We adopt greedy sampling for all methods. All evaluations are conducted on a single NVIDIA H200 GPU.

\begin{figure}[t]
  \begin{center}
    \includegraphics[width=0.7\columnwidth]{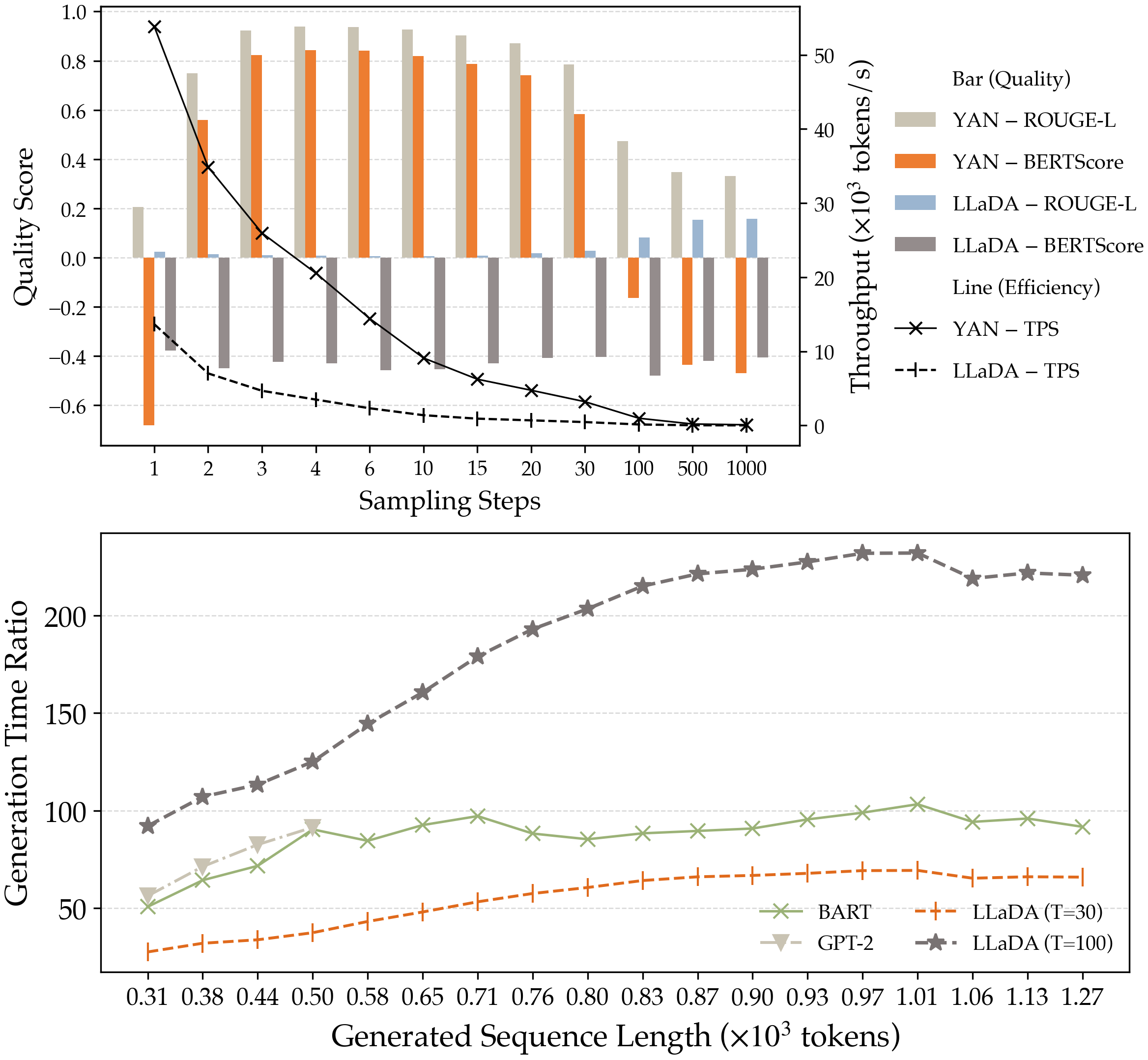}
    \caption{\textbf{Inference efficiency comparison.} \textbf{Top:} Generation quality versus the sampling step $T$ on the text infilling task. \textbf{Bottom:} Inference time across sequence lengths, reported as a ratio relative to \yan\ with $T=3$ (GPT-2 is truncated due to its maximum context length of $512$). In both plots, \yan\ refers to the Transformer-based \yan\ model. \yan\ achieves high-quality generation with as few as three sampling steps, yielding a $40\text{--}50\times$ speedup over \ar\ baselines. In contrast, the diffusion language model \llada\ requires approximately one inference step per token to reach acceptable generation quality (i.e., $T=1000$), resulting in approximately $10^3\times$ higher inference latency for long documents (the $T=1000$ curve is omitted for clarity).}
    \label{fig:efficiency_analysis}
  \end{center}
\end{figure}

\subsection{Language Modeling Capabilities}

At the current model scale, \yan\ demonstrates solid language modeling capabilities, showing strong performance in both text generation and understanding, as shown in Tab.~\ref{tab:main_results}.

\textbf{Language Understanding Performance.} \yan\ consistently outperforms baseline models on classification tasks, achieving near-perfect accuracy on AG News and DBpedia. On the more challenging SST-2 binary sentiment dataset, \yan\ continues to deliver the strongest performance, followed by \llada. In QA tasks, \llada\ attains the highest quality scores, with \yan\ ranking second, partially due to the larger training scale of \llada, which provides richer world knowledge. Nonetheless, the competitive performance of both \llada\ and \yan\ across QA benchmarks indicates that \nar\ models exhibit stronger global text comprehension compared to \ar\ models, particularly outperforming decoder-only architectures such as GPT-2.

\textbf{Text Generation Quality.} \yan\ achieves the highest generation quality on both text infilling and last-word completion tasks. In text infilling, \yan\ outperforms BART, which is trained as a denoising autoencoder. This advantage can be partially attributed to the longer context length used by \yan, which facilitates the modeling of long-range dependencies in the NarrativeQA dataset. For last-word completion, \yan\ surpasses GPT-2 and attains the highest accuracy on SimpleStories, which is longer than ROCStories and thus provides richer contextual information.

\textbf{Transformer versus Mamba.} While \yan\ exhibits competent language modeling capabilities under both architectures, the Transformer-based variant generally outperforms its Mamba-based counterpart. This performance gap suggests that the self-attention mechanism---in particular, bidirectional attention without causal masking in our \nar\ setting---is more effective at capturing contextual information and modeling long-range dependencies. These empirical results are consistent with prior analyses showing that Mamba architectures tend to underperform Transformers on memory-intensive tasks, including information retrieval and long-context understanding \citep{mamba_empirical2024, repeat_after_me2024, decimamba2025}.

\subsection{Inference Efficiency Analysis}

We observe a substantial inference efficiency advantage of \yan\ over both \ar\ and diffusion-based \nar\ models. Fig.~\ref{fig:efficiency_analysis} analyzes inference efficiency as a function of the sampling step $T$ and the generated sequence length on the infilling task, leading to the following key findings. 

\textbf{\yan\ achieves high-quality long-document generation with as few as three Euler sampling steps.} This behavior is also observed across other considered downstream tasks with shorter target sequences, as shown in Tab.~\ref{tab:main_results} and in more detailed results provided in Appendix~\ref{append:sensitivity_to_T}. This stands in sharp contrast to \llada, which requires approximately one inference step per token to reach acceptable generation quality. Moreover, increasing the sampling steps beyond this low-step regime does not yield further performance improvements and is therefore unnecessary.

\textbf{\yan\ achieves a $40\text{--}50\times$ inference speedup over \ar\ baselines and a speedup on the order of $10^3\times$ over the diffusion language model.} The former comparison is made against GPT-2 and BART, which have smaller parameter scales than \yan, indicating that the observed efficiency advantage is not driven by model capacity but instead arises from the parallel decoding with the \moefm\ generator. The latter comparison is against \llada\ with $T=1000$, which is required to attain its highest generation quality, as shown in Fig.~\ref{fig:efficiency_analysis} (see Appendix~\ref{append:time_of_llada} for additional results).

\begin{figure}[ht]
  \begin{center}
    \includegraphics[width=0.7\columnwidth]{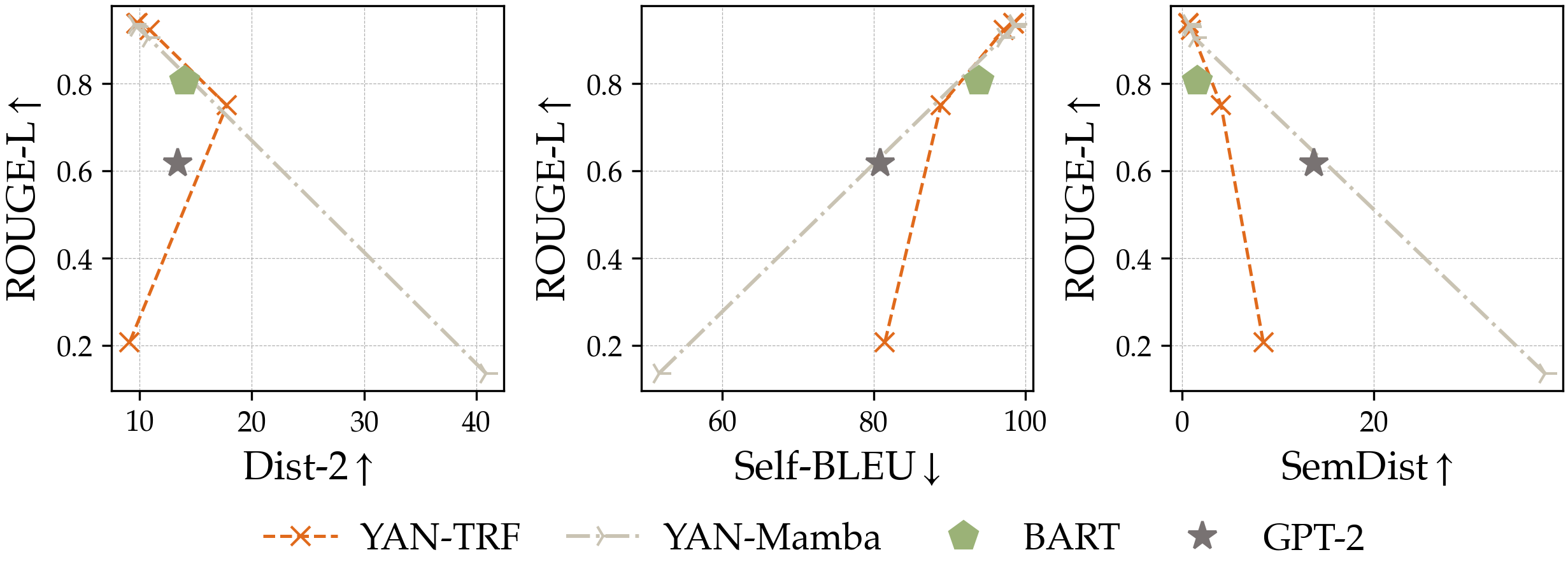}
    \caption{Trade-off between quality and diversity.}
    \label{fig:diversity_tradeoff}
  \end{center}
\end{figure}

\subsection{Diversity Analysis}

While the sampling path follows deterministic ODE trajectories, \yan\ is a valid generative model with stochasticity in the generated text. In this case, diversity becomes an important metric for preventing mode collapse. However, the evaluation of generation diversity is traded off against generation quality in open-ended language generation tasks \citep{diversity_quality_tradeoff2021}. Fig.~\ref{fig:diversity_tradeoff} illustrates the trade-off between quality and diversity on the infilling task under different sampling steps. As shown, configurations that prioritize quality typically correspond to lower diversity. Nevertheless, Fig.~\ref{fig:diversity_tradeoff} also shows configurations whose diversity exceeds that of the baselines, indicating that diversity can be effectively preserved and adjusted through the choice of configuration, such as the number of sampling steps.

\section{Conclusions}

We presented mixture-of-experts flow matching that enhances the latent representational capacity of vanilla flow matching when applied to text. Building on this formulation, we proposed \yan, a non-autoregressive language model using latent flows. \yan\ achieves high-quality generation with substantially fewer decoding steps, leading to faster inference compared to autoregressive and diffusion-based language models. Given its effectiveness at the current model scale, a promising direction for future work is to scale this approach to larger models and datasets. Moreover, while our current implementation adopts a dense mixture-of-experts formulation, exploring sparse expert routing as the model scales may further improve inference efficiency while maintaining generation quality.

\bibliography{paper.bib}
\bibliographystyle{abbrvnat}


\appendix

\section{Mathematical Details}

\subsection{Proof of Proposition~\ref{proposition}}\label{append:derivation1}

The following proposition restates Proposition~\ref{proposition}, with subscripts explicitly indicating the random variables with respect to which the expectations are taken.

\begin{proposition}
The vanilla flow matching objective~\eqref{eq:vfm loss}
$$
\mathcal{L}_\vfm(\psi)=\mathbb{E}_{t,z_0,z_1,z_t}\|u_t^\psi(z_t)-(z_1-z_0)\|^2
$$
is conditionally minimized by 
$$
\hat{u}^\vfm(z_t,t)=\mathbb{E}_{u^*}[u^*\mid z_t,t]=\mathbb{E}_{z_0,z_1}[z_1-z_0\mid z_t, t].
$$
where $u^*=z_1-z_0$. 
\end{proposition}

\begin{proof}
Write $\hat{u}=\hat{u}^\vfm(z_t,t)$ for brevity. Apply decomposition
$$
\|u^\psi-u^*\|^2=\|u^\psi-\hat{u}+\hat{u}-u^*\|^2=\|u^\psi-\hat{u}\|^2+\|\hat{u}-u^*\|^2+2(u^\psi-\hat{u})^\top(\hat{u}-u^*),
$$
where the cross term vanishes in expectation since
$$
\begin{aligned}
    \mathbb{E}_{u^*}\left[(u^\psi-\hat{u})^\top(\hat{u}-u^*)\mid z_t,t\right]&=(u^\psi-\hat{u})^\top\mathbb{E}_{u^*}[\hat{u}-u^*\mid z_t,t]\\
    &=(u^\psi-\hat{u})^\top\left(\hat{u}-\mathbb{E}_{u^*}[u^*\mid z_t,t]\right)=0.
\end{aligned}
$$
Then, the objective becomes 
\begin{equation*}
\begin{aligned}
    \mathcal{L}_\vfm(\psi)&=\mathbb{E}_{t,u^*,z_t}\|u^\psi-u^*\|^2\\
    &=\underbrace{\mathbb{E}_{u^*}\left[\|u^\psi-\hat{u}\|^2\mid z_t,t\right]}_{=\|u^\psi-\hat{u}\|^2}+\underbrace{\mathbb{E}_{u^*}\left[\|\hat{u}-u^*\|^2\mid z_t,t\right]}_{\text{independent of }\psi}+\underbrace{2\mathbb{E}_{u^*}\left[(u^\psi-\hat{u})^\top(\hat{u}-u^*)\mid z_t,t\right]}_{=\,0}\\
    &=\|u^\psi-\hat{u}\|^2+C,
\end{aligned}
\end{equation*}
where $C$ is a constant independent of $\psi$. This is minimized at $u^\psi=\hat{u}$. 
\end{proof}

\subsection{Proof of Theorem~\ref{theorem}}\label{append:derivation2}

\begin{lemma}\label{lemma1}
Introduce an expert assignment random variable $g$ such that the MoE distribution of $u^*$ in Equation~\eqref{eq:moe model}, 
$$
p(u^*\mid z_t,t)= \sum_{k=1}^K\pi_k^\psi(z_t,t)\varphi_\sigma\big(u^*;u_{k,t}^\psi(z_t)\big),
$$
admits the equivalent conditional representation
$$
g\mid z_t,t\sim \mathrm{Cat}\big(\Pi^\psi(z_t,t)\big),\quad u^*\mid g=k,z_t,t\sim \mathcal{N}(u_{k,t}^\psi(z_t),\sigma^2 I),
$$
for $k\in\{1, \dots, K\}$. Here, $\Pi^\psi(z_t,t)=(\pi_1^\psi, \dots, \pi_K^\psi)\in\Delta^{K-1}$ denotes the vector of routings $\pi_k^\psi=\pi_k^\psi(z_t,t)$, and 
$$
\varphi_\sigma(x;\mu)=(2\pi\sigma^2)^{-m/2}\exp\left\{-\frac{1}{2\sigma^2}\|x-\mu\|^2\right\}
$$
denotes the $m$-dimensional Gaussian density function. Then, given an observation of $u^*$, the posterior distribution of $g$ is 
$$
g\mid z_t,t,u^*\sim \mathrm{Cat}\big(\Gamma^\psi(z_t,t,u^*)\big),
$$
where $\Gamma^\psi(z_t,t,u^*)=(\gamma_1^\psi, \dots, \gamma_K^\psi)\in\Delta^{K-1}$ denotes the vector of responsibilities
\begin{equation}\label{eq:gamma}
\gamma_k^\psi=\gamma_k^\psi(z_t,t,u^*)=\frac{\pi_k^\psi\varphi_\sigma\big(u^*;u_{k,t}^\psi(z_t)\big)}{\sum_{k'=1}^K\pi_{k'}^\psi\varphi_\sigma\big(u^*;u_{k',t}^\psi(z_t)\big)},\quad k\in\{1, \dots, K\}.
\end{equation}
\end{lemma}

\begin{proof}
By Bayes' theorem, for each $k\in\{1, \dots, K\}$,
$$
\gamma_k^\psi=\text{Pr}(g=k\mid z_t,t,u^*)=\frac{\text{Pr}(g=k\mid z_t,t)p(u^*\mid z_t,t,g=k)}{\sum_{k'=1}^K \text{Pr}(g=k'\mid z_t,t)p(u^*\mid z_t,t,g=k')},
$$
which simplifies to the stated expression.
\end{proof}

\begin{lemma}\label{lemma2}
Define 
$$
S(u^*;\eta)=\sum_{k=1}^K\pi_k^\psi(z_t,t)\varphi_\sigma\big(u^*;u_{k,t}^\psi(z_t)\big),\quad \ell(\eta)=\mathbb{E}_{u^*}[-\log S(u^*;\eta)\mid z_t,t],
$$
where $u^*=z_1-z_0$ and $\eta=(\pi_1^\psi, \dots, \pi_K^\psi, u_1^\psi, \dots, u_K^\psi)$ denotes the collection of parameters given $(z_t,t)$, with $(\pi_1^\psi, \dots, \pi_K^\psi)\in\Delta^{K-1}$ and $u_k^\psi=u_{k,t}^\psi(z_t)\in\mathbb{R}^m$. Then, conditional on $(z_t,t)$, 
$$
\nabla_\eta\ell(\eta)=\mathbb{E}_{u^*}[-\nabla_\eta\log S(u^*;\eta)\mid z_t,t],
$$
where the gradient with respect to $\pi$ is taken on the simplex $\Delta^{K-1}$, provided that the following regularity conditions hold: 
\begin{enumerate}
    \item[(A1)] There exists $\epsilon\in(0,1/K)$ such that $(\pi_1^\psi, \dots, \pi_K^\psi)\in \Delta_\epsilon^{K-1}=\big\{(\pi_1,\dots,\pi_K):\pi_k\ge \epsilon, \sum_{k=1}^K\pi_k=1\big\}$. 
    \item[(A2)] There exists $B<\infty$ such that $\|u_k^\psi(z_t,t)\|\le B$ for all $k=1, \dots, K$;
    \item[(A3)] $\mathbb{E}_{u^*}[\|u^*\|\mid z_t,t]<\infty$.
\end{enumerate}
\end{lemma}

\begin{proof}
Write $f(u^*;\eta)=-\log S(u^*;\eta)$.
\begin{itemize}
    \item[(i)] \textbf{Differentiability.} Since the Gaussian density $\varphi_\sigma(u^*;\mu_k^\psi)$ is $C^\infty$ with respect to $\mu_k^\psi$, and $S(u^*;\eta)$ is linear in $\pi_k^\psi$ for all $k$, it follows that $f(u^*;\eta)$ is differentiable in $\eta$ for all $u^*$.
    \item[(ii)] \textbf{Gradients.} For $k=1, \dots, K$, 
    $$
    \begin{aligned}
    \nabla_{u_k^\psi} f(u^*;\eta)&=\frac{\pi_k^\psi \varphi_\sigma\big(u^*;u_{k,t}^\psi(z_t)\big)}{S(u^*;\eta)}\frac{u_k^\psi-u^*}{\sigma^2}=\frac{\gamma_k^\psi(u_k^\psi-u^*)}{\sigma^2},\\
    \nabla_{\pi_k^\psi} f(u^*;\eta)&=-\frac{\varphi_\sigma\big(u^*;u_{k,t}^\psi(z_t)\big)}{S(u^*;\eta)}=-\frac{\gamma_k^\psi}{\pi_k^\psi},
    \end{aligned}
    $$
    where $\gamma_k^\psi$ is defined in~\eqref{eq:gamma}. Since $(\pi_1^\psi,\dots,\pi_K^\psi)\in\Delta^{K-1}$, the gradient with respect to $\pi$ is interpreted as the gradient on the simplex, i.e., the orthogonal projection of the Euclidean gradient onto the tangent space
    $$
    T_\pi\Delta^{K-1}=\bigg\{v\in\mathbb{R}^K:\sum_{k=1}^K v_k=0\bigg\}.
    $$
    Consequently, the simplex gradient is
    $$
    \nabla_{\pi_k^\psi}^\Delta f(u^*;\eta)=-\frac{\gamma_k^\psi}{\pi_k^\psi}+\frac{1}{K}\sum_{k'=1}^K\frac{\gamma_{k'}^\psi}{\pi_{k'}^\psi}.
    $$
    
    \item[(iii)] \textbf{Dominating bound.} Since $0\le \gamma_k^\psi\le 1$,
    $$
    \big|\nabla_{\pi_k^\psi}^\Delta f(u^*;\eta)\big|\le\frac{2}{\epsilon}
    $$
    by assumption (A1), and 
    $$
    \|\nabla_{u_k^\psi} f(u^*;\eta)\|\le\frac{\|u_k^\psi-u^*\|}{\sigma^2}\le\frac{\|u_k^\psi\|+\|u^*\|}{\sigma^2}\le\frac{\|u^*\|+B}{\sigma^2}
    $$
    by assumption (A2). Thus, there exists constants $C_0,C_1<\infty$ such that, for all $\eta, u^*$, 
    $$
    \|\nabla_\eta f(u^*;\eta)\|\le C_0+C_1\|u^*\|,
    $$
    and the right-hand side is integrable with respect to $u^*$ by assumption (A3). 

\end{itemize}
Combining (i)-(iii),the result follows from the Dominated Convergence Theorem \citep{rudin1987}.
\end{proof}

With Lemmas~\ref{lemma1} and~\ref{lemma2} in place, the following result completes the proof of Theorem~\ref{theorem}.

\begin{theorem}
The \moefm\ objective~\eqref{eq:moe_loss}
$$
\mathcal{L}_\moefm(\psi)=\mathbb{E}_{t,z_0,z_1,z_t}\bigg[-\log\sum_{k=1}^K\Big\{\pi_k^\psi(z_t,t)\exp\Big(-\frac{1}{2\sigma^2}\|u_{k,t}^\psi(z_t)-(z_1-z_0)\|^2\Big)\Big\}\bigg]
$$
is conditionally minimized by
$$
\begin{aligned}
\hat{\pi}_k^\moefm(z_t,t)&=\mathbb{E}_{z_0,z_1}[\gamma_k^\psi(z_t,t,z_1-z_0)\mid z_t,t],\\
\hat{u}_k^\moefm(z_t,t)&=\frac{\mathbb{E}_{z_0,z_1}[\gamma_k^\psi(z_t,t,z_1-z_0)(z_1-z_0) \mid z_t,t]}{\mathbb{E}_{z_0,z_1}[\gamma_k^\psi(z_t,t,z_1-z_0)\mid z_t,t]},
\end{aligned}
$$
for $k\in\{1, \dots, K\}$, where $\gamma_k^\psi$ is the responsibility defined in~\eqref{eq:gamma}.
\end{theorem}

\begin{proof}
By the Law of Iterated Expectations, 
$$
\mathcal{L}_\moefm(\psi)=\mathbb{E}_{z_t,t}\big[\mathbb{E}_{z_0,z_1}(-\log S(z_1-z_0,\eta)\mid z_t,t)\big]=\mathbb{E}_{z_t,t}[\ell(\eta)]
$$
where $\eta$, $S(\cdot)$, and $\ell(\cdot)$ are defined in Lemma~\ref{lemma2}. Now we show that $\hat{\pi}_k=\hat{\pi}_k^\moefm(z_t,t)$ and $\hat{u}_k=\hat{u}_k^\moefm(z_t,t)$ minimize $\ell(\eta)$.

\begin{itemize}
    \item[(i)] By Lemma~\ref{lemma2} and its proof,
    $$
    \nabla_{u_k^\psi}\ell(\eta)=\mathbb{E}_{u^*}\big[\nabla_{u_k^\psi}f(u^*;\eta)\mid z_t,t\big]=\mathbb{E}_{u^*}\big[\gamma_k^\psi(u_k^\psi-u^*)/\sigma^2\mid z_t,t\big].
    $$
    Let $\nabla_{u_k^\psi}\ell(\eta)=0$, then we have minimizer
    $$
    \hat{u}_k=\frac{\mathbb{E}_{u^*}[\gamma_k^\psi u^*\mid z_t,t]}{\mathbb{E}_{u^*}[\gamma_k^\psi\mid z_t,t]}.
    $$
    
    \item[(ii)] Define Lagrangian \citep{boyd2004}
    $$
    \mathcal{J}=\ell(\eta)+\lambda\bigg(\sum_{k=1}^K\pi_k^\psi-1\bigg)
    $$
    for multiplier $\lambda\in\mathbb{R}$. Let 
    $$
    \nabla_{\pi_k^\psi}\mathcal{J}=\mathbb{E}_{u^*}\bigg[-\frac{\gamma_k^\psi}{\pi_k^\psi}\mid z_t,t\bigg]+\lambda=0,\quad k=1, \dots, K.
    $$
    Since $\sum_{k=1}^K\pi_k^\psi=1$ and $\sum_{k=1}^K\gamma_k^\psi=1$, we have $\lambda=1$ and minimizer
    $$
    \hat{\pi}_k=\mathbb{E}_{u^*}[\gamma_k^\psi\mid z_t,t].
    $$
\end{itemize}

\end{proof}

\subsection{Two Extrema of $\sigma$}\label{append:derivation3}

\subsubsection{$\sigma\to 0$}

\begin{proposition}
When $\sigma\to 0$, 
$$
\begin{aligned}
\hat{\pi}_k^\moefm(z_t,t)&\to\mathbb{E}_{u^*}\big[\mathbbm{1}\{k\in\mathcal{M}(u^*)\}\pi_k^{\mathcal{M},\psi}(u^*)\mid z_t,t\big],\\
\hat{u}_k^\moefm(z_t,t)&\to \frac{\mathbb{E}_{u^*}\big[\mathbbm{1}\{k\in\mathcal{M}(u^*)\}\pi_k^{\mathcal{M},\psi}(u^*)u^*\mid z_t,t\big]}{\mathbb{E}_{u^*}\big[\mathbbm{1}\{k\in\mathcal{M}(u^*)\}\pi_k^{\mathcal{M},\psi}(u^*)\mid z_t,t\big]},
\end{aligned}
$$
for $k=1, \dots, K$, where $\mathcal{M}(u^*)=\mathop{\arg\min}_{1\le k\le K}\|u_k^\psi-u^*\|^2$ is the set of minimizing indices and $\pi_k^{\mathcal{M},\psi}(u^*)=\pi_k^\psi/\sum_{k'\in\mathcal{M}(u^*)}\pi_{k'}^\psi$ is the scaled routing probability on set $\mathcal{M}(u^*)$. 
\end{proposition}

\begin{proof}
Define $d_k=\|u_k^\psi-u^*\|^2$ and $d_{\min}=\min_{1\le k\le K}d_k$. The responsibilities~\eqref{eq:gamma} can be expressed as 
$$
\gamma_k^\psi=\frac{\pi_k^\psi\exp\big\{-\frac{1}{2\sigma^2}(d_k-d_{\min})\big\}}{\sum_{k'=1}^K\pi_{k'}^\psi\exp\big\{-\frac{1}{2\sigma^2}(d_{k'}-d_{\min})\big\}},\quad k\in\{1, \dots, K\}.
$$
Since $d_k\ge d_{\min}$ for all $k$, then, as $\sigma\to 0$, the denominator is 
$$
\sum_{k'\in\mathcal{M}}\pi_{k'}^\psi+\sum_{k'\notin\mathcal{M}}\pi_{k'}^\psi\exp\Big\{-\frac{1}{2\sigma^2}(d_{k'}-d_{\min})\Big\}\to \sum_{k'\in\mathcal{M}}\pi_{k'}^\psi,
$$
while the numerator
$$
\pi_k^\psi\exp\Big\{-\frac{1}{2\sigma^2}(d_k-d_{\min})\Big\}\to
\begin{cases}
    \pi_k^\psi,&k\in\mathcal{M},\\
    0,&k\notin\mathcal{M}.
\end{cases}
$$
Thus, as $\sigma\to 0$
$$
\gamma_k^\psi\to \mathbbm{1}\{k\in\mathcal{M}\}\frac{\pi_k^\psi}{\sum_{k'\in\mathcal{M}}\pi_{k'}^\psi}= \mathbbm{1}\{k\in\mathcal{M}\}\pi_k^{\mathcal{M},\psi}. 
$$
This leads to the stated expressions under regularity assumptions in Lemma~\ref{lemma2}.

\end{proof}

\begin{corollary}
Assume singleton $\mathcal{M}(u^*)=\{k^*(u^*)\}$. Then, as $\sigma\to 0$, 
$$
\begin{aligned}
    \hat{\pi}_k^\moefm(z_t,t)&\to \text{Pr}\big(k^*(u^*)=k\mid z_t,t\big),\\
    \hat{u}_k^\moefm(z_t,t)&\to \mathbb{E}_{u^*}[u^*\mid k^*(u^*)=k, z_t, t].
\end{aligned}
$$
for $k=1, \dots, K$. 
\end{corollary}

\textbf{Interpretations.} As $\sigma\to 0$, the optimal routing $\hat{\pi}_k^\moefm(z_t,t)$ converges to the conditional probability that expert $k$ is selected by the \textit{hard assignment rule} $k^*(u^*)=\arg\min_{1\le k\le K}\|u_k^\psi-u^*\|^2$. Meanwhile, each expert vector field $\hat{u}_k^\moefm(z_t,t)$ converges to the conditional expectation of the velocity target given that it is assigned to expert $k$. Therefore, \moefm\ reduces to a hard-assignment MoE formulation that performs conditional vector field estimation.

\subsubsection{$\sigma\to\infty$}

As $\sigma\to\infty$, the \moefm\ likelihood~\eqref{eq:moe_loss} $\mathcal{L}_\moefm(\psi)$ converges to a constant independent of $\psi$, being uninformative in the parameter $\psi$. Also, the responsibilities~\eqref{eq:gamma} converge to the routings. In other words, observing $u^*$ provides no guidance for optimization, and the expert assignment is \textit{non-identifiable} under the \moefm\ objective---no assignment is preferred over another by the objective.

\subsection{A Full Expression of \mmd\ Regularizer}\label{append:mmd}

Let $\{x_i\}_{i=1}^n \sim P$ and $\{y_j\}_{j=1}^m \sim Q$ be independent and identically distributed samples drawn from distributions $P$ and $Q$, respectively. We use the unbiased empirical estimator of the squared Maximum Mean Discrepancy (\mmd):
$$
\widehat{\mmd}_\kappa^2(P,Q)=\frac{1}{n(n-1)} \sum_{i\neq j} \kappa(x_i,x_j)+\frac{1}{m(m-1)} \sum_{i\neq j} \kappa(y_i,y_j)-\frac{2}{nm} \sum_{i=1}^n \sum_{j=1}^m \kappa(x_i,y_j),
$$
where $\kappa(x,x')=\sum_{s\in\mathcal{S}}\kappa_s(x,x')$ is a sum of radial basis function kernels $\kappa_s(x,x')=\exp\left\{-\|x-x'\|^2/(2s^2)\right\}$ and $\mathcal{S}=\{0.2, 0.5, 1.0, 2.0, 5.0\}$ denotes a set of kernel bandwidths.

\subsection{A discussion of the \yan\ objective}\label{append:ce_supervision}

The second stage of the \yan\ model trains the flow matching under the loss function~\eqref{eq:yan_objective}, i.e., 
$$
\mathcal{L}_\yan(\Theta)=\alpha_\moefm\mathcal{L}_\moefm(\Theta)+\alpha_\ce\mathcal{L}_\ce(\Theta). 
$$
Although it might seem counterintuitive to include the CE loss in training the flow, as CE is not required by the flow formulation, an interesting observation from our implementation is that the CE term plays a crucial role in effectively training the latent flow. With $\mathcal{L}_\moefm$ alone, the learned flow suffers from misalignment issues, where the latent distribution is well modeled but decodes to incorrect tokens. We interpret this behavior as arising from the token-label-agnostic nature of the NLL objective, whereas incorporating the CE loss anchors the latent flow to the decoding objective. Indeed, CE supervision is common in \nar\ modeling, despite being motivated by different considerations (e.g., \citeauthor{diffulamma2025}, \citeyear{diffulamma2025}; \citeauthor{cadd2025}, \citeyear{cadd2025}).

\section{Implementation Details}\label{append:implementation}

\subsection{Dataset Statistics}\label{append:training_data}

\begin{table}[ht]
\centering
\setlength{\tabcolsep}{8pt}
\caption{Summary of datasets used for pre-training and fine-tuning, including dataset size and per-sample word/token length statistics.}
\small
\begin{tabular}{llccc|ccc}
\toprule
\multirow{2}{*}{\textbf{Datasets}} & \multirow{2}{*}{\textbf{Size}} & \multicolumn{3}{c}{\textbf{Words}} & \multicolumn{3}{c}{\textbf{Tokens}} \\
\cmidrule(lr){3-5} \cmidrule(lr){6-8}
 &  & \textbf{Min} & \textbf{Mean} & \textbf{Max} & \textbf{Min} & \textbf{Mean} & \textbf{Max} \\
\midrule
FineWiki + FineWeb & 8.7M & 7 & 501.16 & 1770 & 20 & 769.67 & 2048 \\
\addlinespace[2pt]
NarrativeQA & 46.8k & 212 & 600.74 & 1048 & 267 & 765.88 & 1406 \\
\addlinespace[2pt]
ROCStories & 98.2k & 19 & 43.96 & 71 & 26 & 53.19 & 88 \\
\addlinespace[2pt]
SimpleStories & 2.1M & 42 & 225.77 & 606 & 55 & 280.65 & 754 \\
\addlinespace[2pt]
bAbI & 11k & 17 & 38.05 & 65 & 32 & 59.33 & 92 \\
\addlinespace[2pt]
SQuAD & 98.2k & 34 & 139.53 & 599 & 56 & 191.53 & 787 \\
\addlinespace[2pt]
AG News & 127.6k & 17 & 43.75 & 162 & 31 & 64.47 & 227 \\
\addlinespace[2pt]
DBpedia & 630k & 9 & 54.69 & 226 & 23 & 86.48 & 303 \\
\addlinespace[2pt]
SST-2 & 68.2k & 5 & 12.89 & 49 & 15 & 24.94 & 69 \\
\bottomrule
\end{tabular}
\label{tab:datasets}
\end{table}

Tab.~\ref{tab:datasets} summarizes the datasets used for pre-training and fine-tuning, reporting total sample size and per-sample word and token length statistics (min / mean / max). Word lengths are measured as the number of words per sample, and token lengths are computed after tokenization. For each dataset, we take a 90\%-5\%-5\% split for training, validation, and testing. Tab.~\ref{tab:training_tokens} shows the total number of tokens observed during pre-training, estimated by $\text{Total Tokens}=\text{Training Steps}\times \text{Global Batch Size}\times\text{Average Tokens per Sample}$. 

\begin{table}[ht]
\centering
\caption{Training steps and tokens.}
\small
\begin{tabular}{lccc}
\toprule
Model & Training Steps & Global Batch Size & Total Tokens \\
\midrule
\yan-Transformer & 820k & 32 & 20.2B \\
\yan-Mamba & 600k & 32 & 14.8B \\
\bottomrule
\end{tabular}
\label{tab:training_tokens}
\end{table}

\subsection{Datasets and Evaluation Details}\label{append:evaluation_details}

\textbf{FineWiki} \citep{finewiki_data} and \textbf{FineWeb} \citep{fineweb_data} are high-quality large-scale corpora that have been used in recent \nar\ work for pre-training (e.g., \citeauthor{diffulamma2025}, \citeyear{diffulamma2025}). In particular, FineWeb improves upon the commonly used OpenWebText dataset \citep{openwebtext_data}. We use its subset, FineWeb-Edu, which contains educational content with dense factual and conceptual knowledge \citep{fineweb_edu_data}. 

\textbf{NarrativeQA} \citep{narrativeqa_data} consists of stories paired with corresponding questions and answers for evaluating document-level understanding. It provides long documents with rich long-range dependencies and lexical diversity; therefore, we leverage it for text infilling evaluation by randomly masking 5\%–10\% of tokens, with masking lengths uniformly sampled from ${1,2,3}$. 

\textbf{ROCStories} \citep{rocstories_data} comprises daily-life short stories. We use it for the last-word completion task, following Misra et al. (\citeyear{rocstories_last_completion2}) and Amara et al. (\citeyear{rocstories_last_completion1}). At a larger scale and with longer contexts, we similarly use the \textbf{SimpleStories} dataset \citep{simplestories2025}, which contains stories generated by GPT-4o-mini. For the last-word completion task, we treat the text excluding the final token as the source input and the final token as the target output. 

\textbf{bAbI} \citep{babi_data1, babi_data2} is a reading comprehension dataset consisting of various question–answering tasks, including counting, lists/sets, and argument relations. \textbf{SQuAD} \citep{squad_f1_2016} is another reading comprehension dataset in which the answer to each question is a span extracted from the corresponding passage. 

\textbf{AG News} and \textbf{DBpedia} \citep{agnews_dbpedia_data} contain news articles and Wikipedia articles, respectively, categorized into 4 and 14 classes for classification evaluation. \textbf{SST-2} \citep{sst2_data} is designed for sentence-level sentiment classification and is included in the GLUE benchmark \citep{glue_evaluation}. For question-answering and classification tasks, we treat the context passage (and question, when applicable) as the source input and the answer or class label as the target output.

\begin{figure*}[ht]
  \begin{center}
    \centerline{\includegraphics[width=\textwidth]{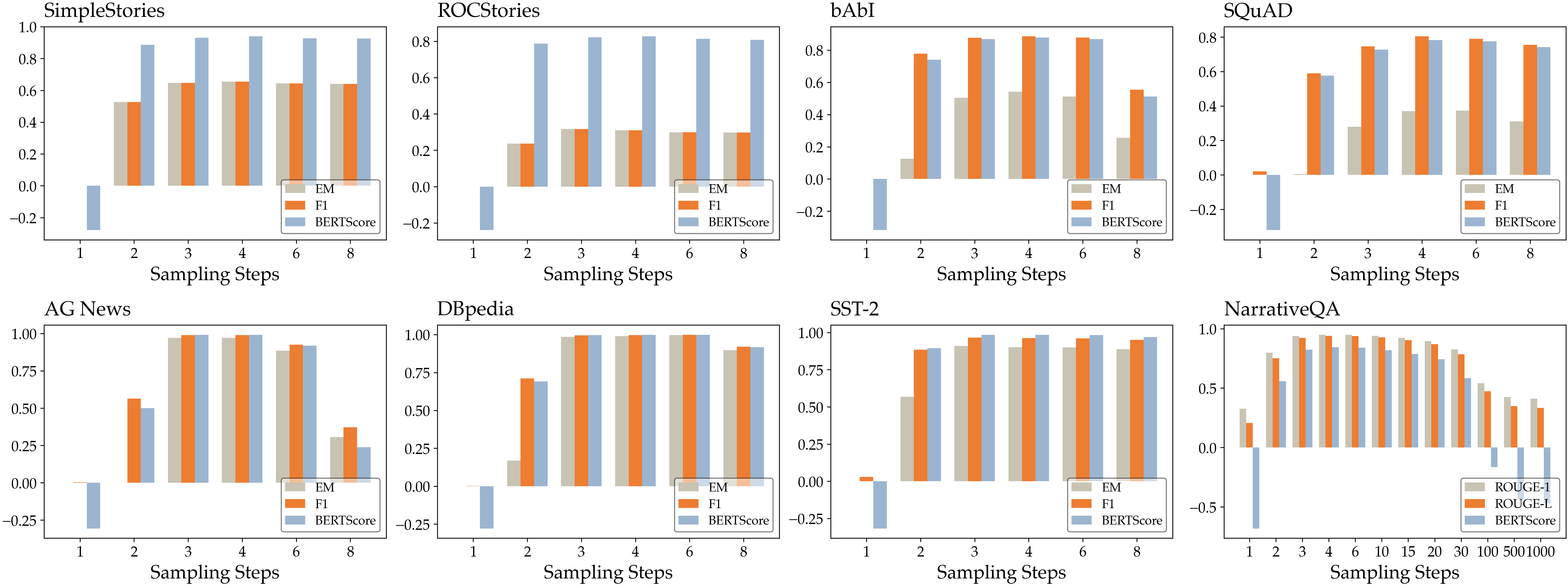}}
    \caption{Effect of sampling steps on generation quality of \yan\ (Transformer) across tasks.}
    \vskip -0.05in
    \label{fig:bar_line_single_combine_yan}
  \end{center}
\end{figure*}

\begin{figure*}[ht]
  \begin{center}
    \centerline{\includegraphics[width=\textwidth]{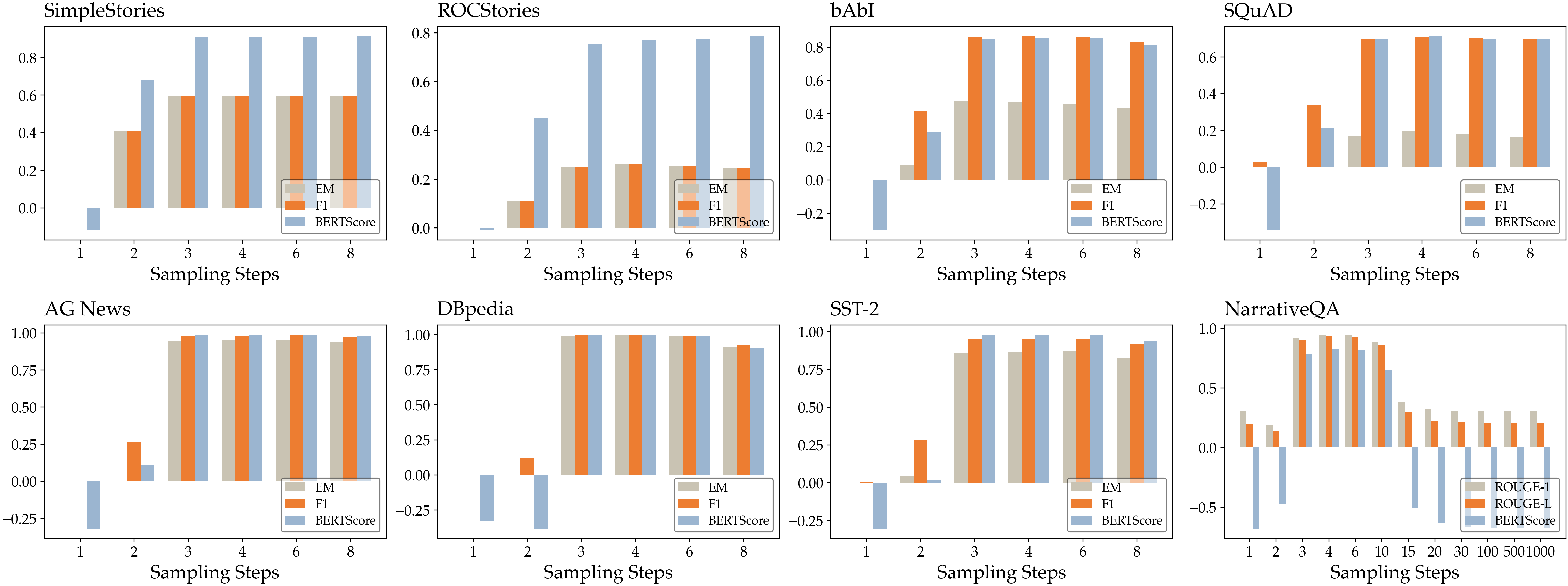}}
    \caption{Effect of sampling steps on generation quality of \yan\ (Mamba) across tasks.}
    \vskip -0.05in
    \label{fig:bar_line_single_combine_mamba}
  \end{center}
\end{figure*}

\begin{figure*}[ht]
  \begin{center}
    \centerline{\includegraphics[width=\textwidth]{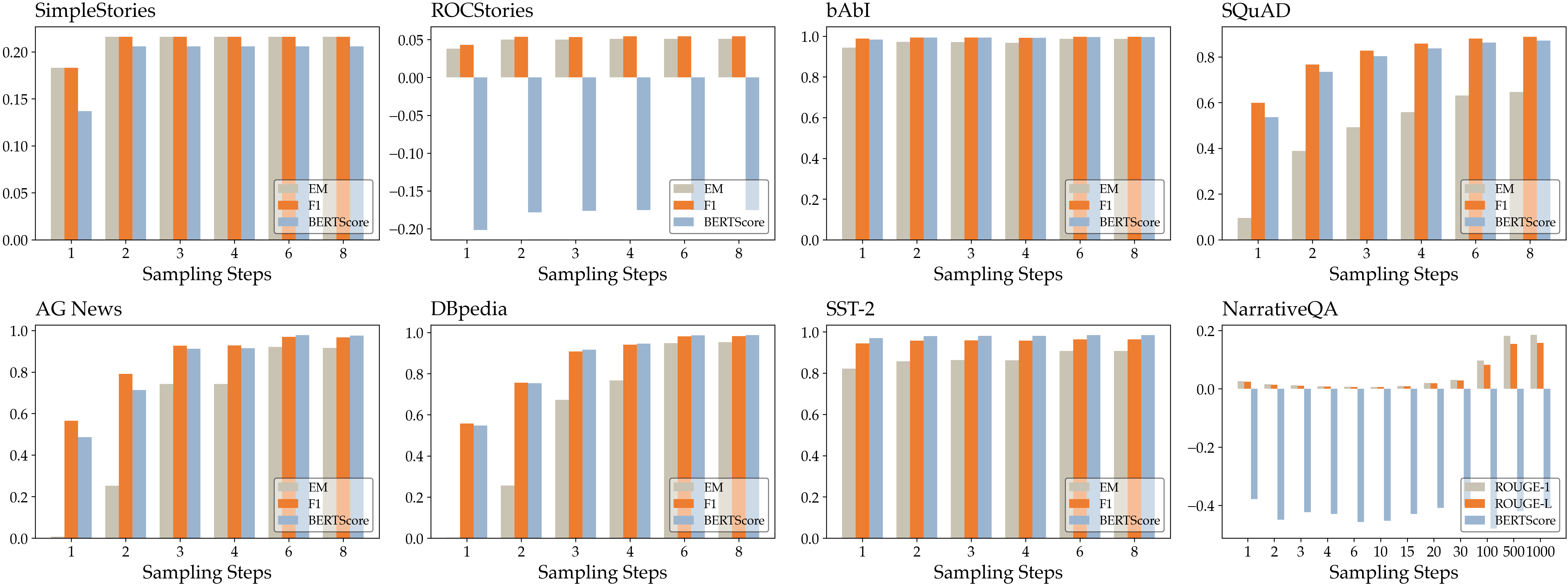}}
    \caption{Effect of sampling steps on generation quality of \llada\ across tasks.}
    \vskip -0.05in
    \label{fig:bar_line_single_combine_llada}
  \end{center}
\end{figure*}

\subsection{Model Optimization and Hyperparameters}\label{append:architecture_details}

\textbf{Transformer.} The encoder of \yan\ uses 6 transformer blocks \citep{transformer2017}. The vector field network uses 8 transformer blocks, with cross attention to the encoder at the 1st, 2nd, 3rd, 5th layers. The gating network is a 2-layer MLP. The decoder network consists of two linear layers with a skip connection (i.e., residual feed-forward network). During training, ODE integration is performed using four Euler steps. The maximum context length is 2048, and the hidden dimension of the model is 512. We use 6 experts, with $\sigma=0.1$.

\textbf{Mamba.} The Mamba-based \yan\ follows a similar architecture to its Transformer-based counterpart, except that within each Transformer block, the self-attention mechanism is replaced by a bidirectional Mamba token mixer. Specifically, two independently parameterized Mamba modules are applied to the input sequence in forward and reverse temporal order, and their outputs are concatenated and linearly projected back to the model dimension. The encoder uses 4 such blocks, and the vector field network uses 6 such blocks.

\textbf{Tokenizer.} \yan\ uses the LLaMA~3 tokenizer \citep{llama2024} with a vocabulary size of 128256. Since \yan\ handles variable-length inputs via right padding, we additionally include a special token \texttt{<|padding|>} with \texttt{tokenid=128256}. To support the infilling task, we further add a special token \texttt{<|mask|>} with \texttt{tokenid=128257}. The resulting vocabulary size is 128258. For all other baseline models evaluated on downstream tasks, we use their original tokenizers.

\textbf{Optimization.} We use the AdamW optimizer \citep{adamw2019} with $\beta_1=0.9, \beta_2=0.999$, and a weight decay of $0.01$. The learning rate is set to $1e{-4}$ for pre-training and $1e{-5}$ for fine-tuning. We apply a linear warmup schedule for the first 1000 steps starting from 0, after which the learning rate remains constant at the target value. Training is performed on a single node using 8 NVIDIA H200 GPUs.

\subsection{Fine-Tuning}\label{append:finetune}

BART is a sequence-to-sequence model that treats the source and target texts in the same way as \yan. For \llada\ and GPT-2, the source is treated as a prefix concatenated to the target. Across all models, we use oracle sequence lengths for both training and inference, excluding the impact of different length-handling strategies.

\section{Additional Results}

\subsection{Sensitivity to Sampling Steps}\label{append:sensitivity_to_T}

As shown in Fig.~\ref{fig:bar_line_single_combine_yan} and Fig.~\ref{fig:bar_line_single_combine_mamba}, a sampling step of $T=3$ is sufficient for \yan\ to achieve acceptable generation quality with both Transformer and Mamba architectures. Increasing $T$ beyond $6$ generally degrades generation quality, suggesting that excessive sampling has adverse effects. This behavior contrasts sharply with the diffusion language model \llada\ \citep{llada2025}, for which increasing the number of sampling steps typically improves generation quality, as illustrated in Fig.~\ref{fig:bar_line_single_combine_llada}.

\subsection{Inference Time of \llada}\label{append:time_of_llada}

\begin{figure*}[ht]
  \begin{center}
    \centerline{\includegraphics[width=0.55\textwidth]{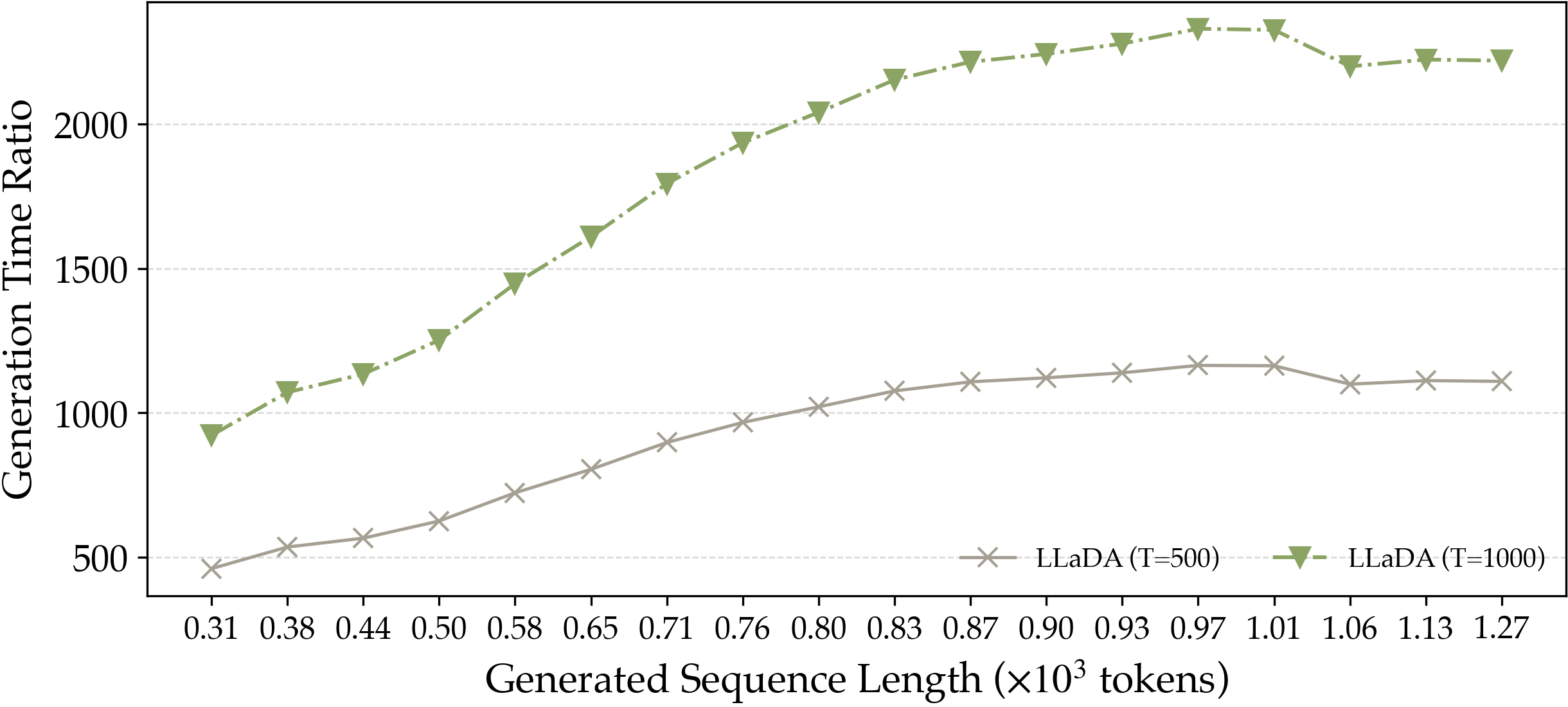}}
    \caption{Inference time across generated sequence lengths, reported as a ratio to \yan\ (Transformer) with $T=3$.}
    \vskip -0.05in
    \label{fig:times_versus_lengths_ratio_T1000}
  \end{center}
\end{figure*}

Fig.~\ref{fig:times_versus_lengths_ratio_T1000} shows the inference time curves of \llada\ with sampling steps $T=500$ and $T=1000$, complementing Fig.~\ref{fig:efficiency_analysis}. It demonstrates the inference latency on the order of $10^3\times$ higher than that of \yan.

\subsection{Generation Examples}\label{append:showcase}

Fig.~\ref{fig:display_last} and Fig.~\ref{fig:display_infill} present selected examples generated by \yan\ for three downstream tasks: question answering, last-word completion, and text infilling. Each example is generated using an Euler ODE solver with four steps.

\begin{figure*}[hp]
  \vskip -0.1in
  \begin{center}
    \centerline{\includegraphics[width=\textwidth]{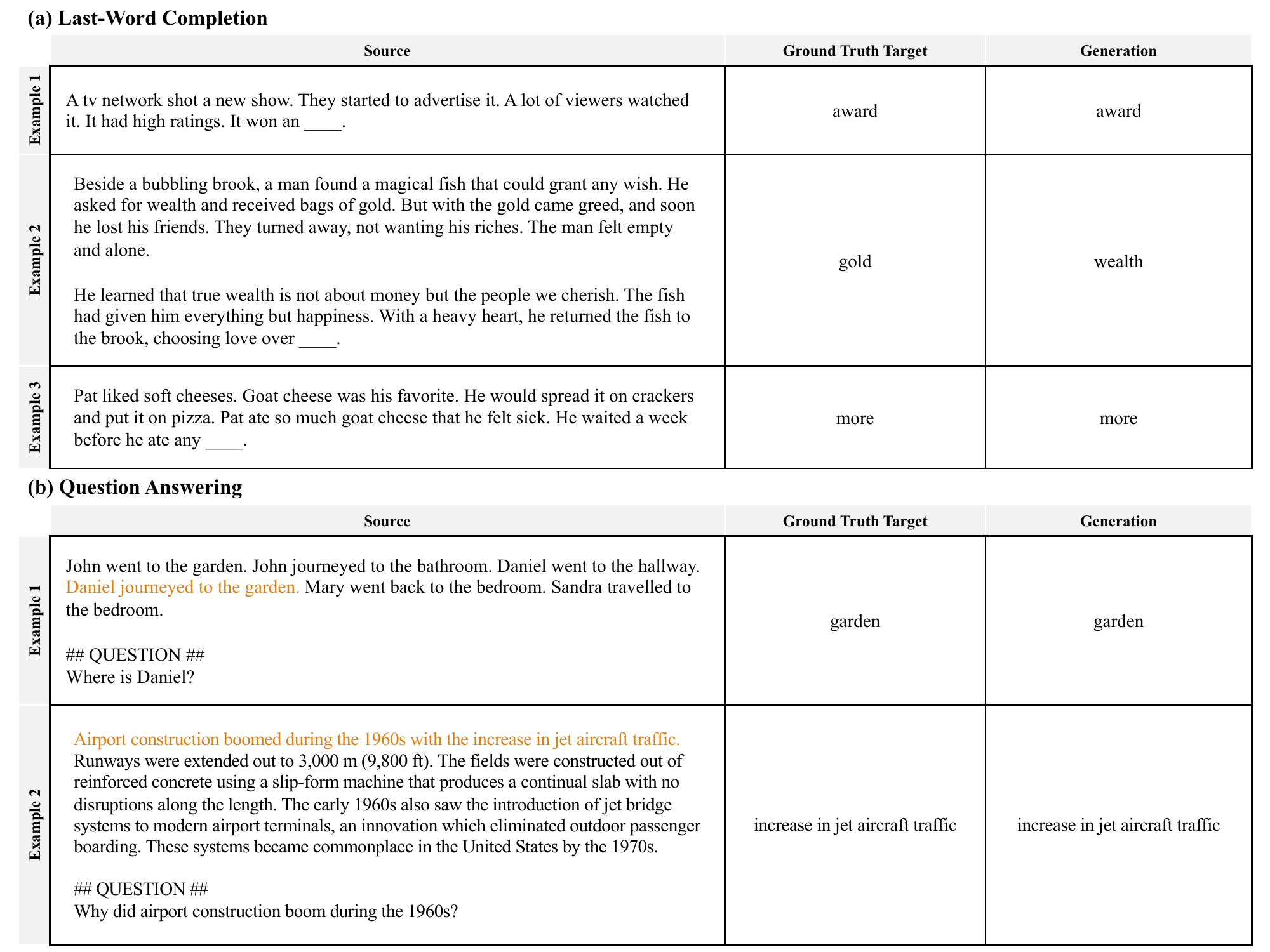}}
    \vskip -0.1in
    \caption{Last-word completion and question answering examples of \yan}
    \label{fig:display_last}
  \end{center}
\end{figure*}

\begin{figure*}[hp]
  \begin{center}
    \centerline{\includegraphics[width=\textwidth]{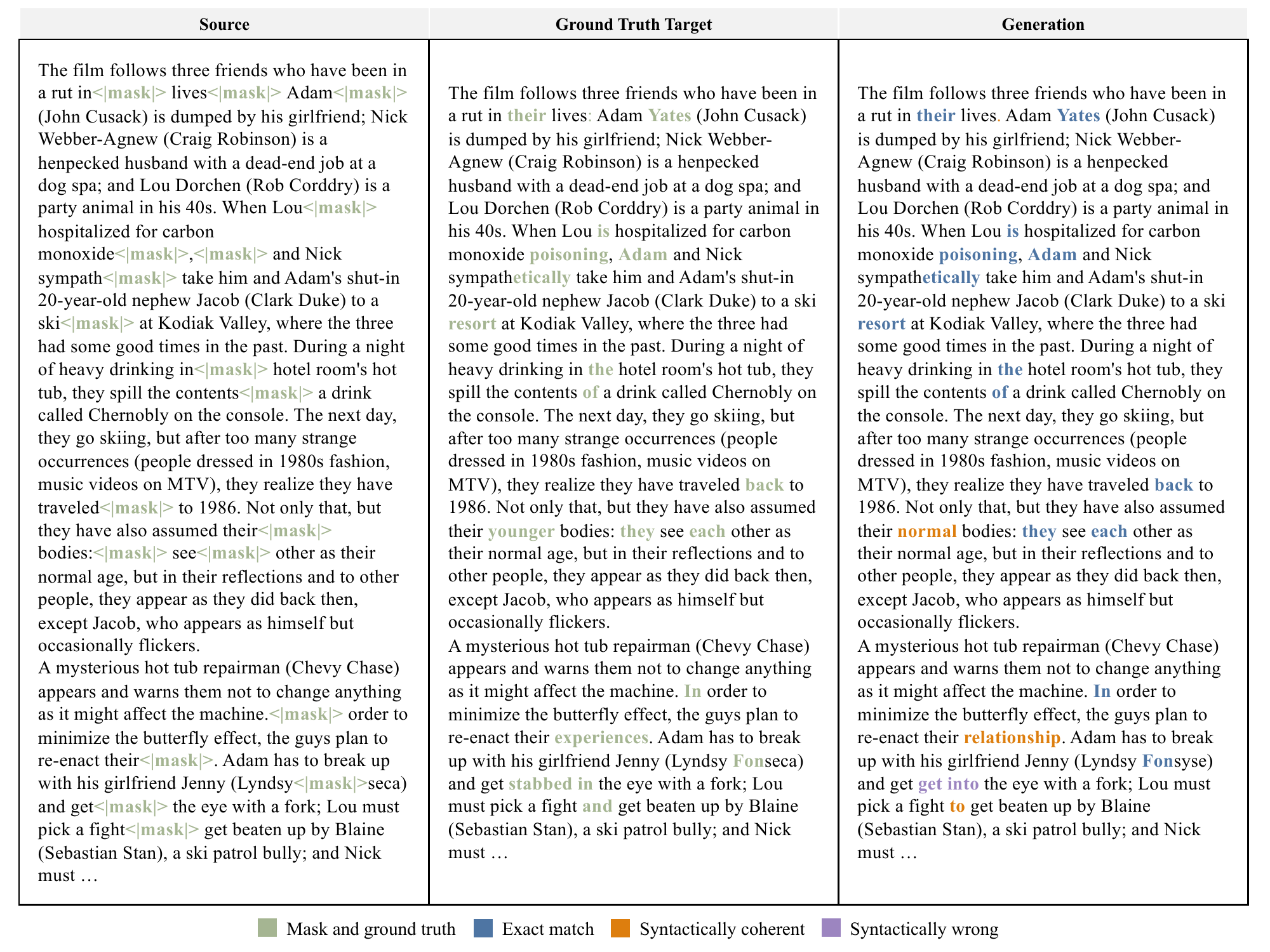}}
    \vskip -0.05in
    \caption{A text infilling example of \yan.}
    \label{fig:display_infill}
  \end{center}
\end{figure*}



\end{document}

%% file: table.tex
\begin{table*}[t]
\centering
\caption{Quality and efficiency results. For metrics with multiple values, results are reported following the order specified in the \emph{Metric} column. $T^*$ denotes the optimal sampling step that yields the best performance on each dataset, and the reported metric values correspond to $T^*$. R-1, R-L, and BS-F1 denote ROUGE-1, ROUGE-L, and BERTScore-F1, respectively. \textbf{Bold} and \underline{underlined} values indicate the best and second-best results, respectively. \yan-M and \yan-TRF denote \yan\ with Mamba and Transformer architectures, respectively.}
\label{tab:main_results}
\footnotesize
\begin{tabularx}{\textwidth}{>{\centering\arraybackslash}X >{\centering\arraybackslash}X >{\centering\arraybackslash}X >{\centering\arraybackslash}X >{\centering\arraybackslash}X >{\centering\arraybackslash}X >{\centering\arraybackslash}X}
\toprule
\multicolumn{2}{@{}c@{}}{\textbf{Model Info}} & \textbf{GPT-2} & \textbf{BART} & \textbf{LLaDA} & \textbf{YAN-M} & \textbf{YAN-TRF} \\
\midrule
\multicolumn{2}{@{}c@{}}{Size} & 124M & 139M & 8B & 210M & 200M \\
\multicolumn{2}{@{}c@{}}{$T^*$} & - & - & 1k/2/4/6/8/6/8/6 & 4/4/4/4/4/6/4/4 & 4/4/4/3/4/3/4/4 \\
\specialrule{0.06em}{3pt}{3pt}
\textbf{Dataset} & \textbf{Metric} &  &  &  &  &  \\
\midrule
NarrativeQA & R-1/R-L/TPS & 68.9/61.8/206 & 81.8/80.6/211 & 18.5/15.8/14 & \underline{94.6}/\underline{93.6}/\underline{18.1k} & \textbf{94.9}/\textbf{93.9}/\textbf{20.6k} \\
SimpleStories & EM/BS-F1 & 42.8/85.1 & 46.5/90.3 & 21.6/20.6 & \underline{59.7}/\underline{91.1} & \textbf{65.5}/\textbf{93.9} \\
ROCStories & EM/BS-F1 & \underline{28.5}/\underline{79.7} & 21.3/70.6 & 5.1/-17.5 & 26.1/77.0 & \textbf{31.0}/\textbf{82.7} \\
AG News & Accuracy & 93.8 & 91.2 & 92.1 & \underline{95.1} & \textbf{97.2} \\
DBpedia & Accuracy & 98.9 & 94.7 & 95.3 & \textbf{99.5} & \underline{99.1} \\
SST-2 & Accuracy & 90.1 & 88.0 & \underline{90.7} & 87.4 & \textbf{91.0} \\
SQuAD & F1/BS-F1 & 48.0/41.0 & 78.9/76.7 & \textbf{88.8}/\textbf{87.2} & 70.8/71.3 & \underline{80.4}/\underline{78.2} \\
bAbI & F1/BS-F1 & 47.7/15.7 & 78.3/74.8 & \textbf{99.7}/\textbf{99.6} & 86.4/85.3 & \underline{88.5}/\underline{87.8} \\
\bottomrule
\end{tabularx}
\end{table*}